\documentclass[12pt]{article}

\usepackage[breaklinks=true]{hyperref}
\hypersetup{colorlinks=true, citecolor=blue, linkcolor=blue}
\usepackage[margin=.85in]{geometry}

\usepackage{natbib}
\usepackage{amsmath,amsthm,amssymb,bbm}

\renewcommand{\ge}{\geqslant}
\renewcommand{\le}{\leqslant}

\newcommand{\Om}{\mathbf\Omega}
\newcommand{\I}[1]{\mathbf 1_{\left\{#1\right\}}}
\newcommand{\Exp}[1]{\exp\left(#1\right)}

\newtheorem{tm}{Theorem}[section]
\newtheorem{ex}[tm]{Example}
\newtheorem{re}[tm]{Remark}

\newtheorem{lm}[tm]{Lemma} 
\newtheorem{co}[tm]{Corollary}

\theoremstyle{definition}
\newtheorem{df}[tm]{Definition}

\newcommand{\X}{\mathbf X}

\newcommand{\wc}{\widetilde{c}}

\renewcommand{\mid}{\lvert}

\renewcommand{\c}{\mathbf c}
\newcommand{\x}{\mathbf x}

\renewcommand{\P}{\mathop{{}\mathbb{P}}}
\newcommand{\E}{\mathop{{}\mathbb{E}}}

\newcommand{\EE}[1]{\mathop{{}\mathbb{E}}\left[ #1 \right] }

\newcommand{\F}{\mathcal F}
\newcommand{\T}{\mathcal T}
\newcommand{\A}{\mathcal A}

\newcommand{\tpw}{\textsf{tpw}}

\newcommand{\Log}[1]{\log\left(#1\right)}
\usepackage{cancel}
\newcommand{\comment}[1]{}
\renewcommand{\S}{\mathbf S}
\newcommand{\R}{\mathbb R}
\newcommand{\D}{\mathcal D}
\newcommand{\RH}{\widehat R}

\newcommand{\norm}[1]{\left\lVert#1\right\rVert}

\newcommand{\mra}[1]{#1}
\newcommand{\rui}[1]{#1}

\newcommand{\red}[1]{{\color{red}{#1}}}

\renewcommand{\red}[1]{#1}

\numberwithin{equation}{section}

\newcommand{\redd}[1]{{\color{red}{#1}}}
\renewcommand{\redd}[1]{#1}

\usepackage{enumitem}

\def\dfrac#1#2{
\hbox{\large$\textstyle\frac{#1}{#2}$}}

\theoremstyle{thmstyleone}%

\usepackage[width=.85\textwidth]{caption}

\theoremstyle{thmstyletwo}%

\theoremstyle{thmstylethree}%

\usepackage{float}

\usepackage{tikz}
\usetikzlibrary{calc, positioning, fit, shapes.misc}
\tikzstyle{edgeBIG} = [gray, line cap=round, line join=round, line width=14pt]
\pgfdeclarelayer{background}
\pgfsetlayers{background,main}

\usetikzlibrary{arrows,patterns}
\usetikzlibrary{shapes.geometric}
\usepackage{subcaption}

\begin{document}

\title{\bf\Large Generalization bounds for learning under graph-dependence: A survey}

\author
{\normalsize 
Rui-Ray Zhang \\
\normalsize 
rui.zhang@bse.eu \\
\normalsize 
Barcelona School of Economics \\
\normalsize School of Mathematics,
Monash University
\and
\normalsize 
Massih-Reza Amini\\
\normalsize 
massih-reza.amini@imag.fr\\
\normalsize 
LIG/CNRS, University Grenoble Alpes}
\date{}

\maketitle

\begin{abstract}
Traditional statistical learning theory relies on the assumption that data are identically and independently distributed (i.i.d.).
\redd{However, this assumption often does not hold in many real-life applications.
In this survey, we explore learning scenarios where examples are dependent and their dependence relationship 
is described by a {\it dependency graph}, a commonly utilized model in probability and combinatorics.}
We collect various graph-dependent concentration bounds, which are then used to derive Rademacher complexity and stability generalization bounds for learning from graph-dependent data. We illustrate this paradigm through practical learning tasks 
and provide some research directions for future work. To our knowledge, this survey is the first of this kind on this subject.
\end{abstract}


\section{Introduction}

\mra{The central assumption in machine learning is that observations are independently and identically distributed (i.i.d.) with respect to a fixed yet unknown probability distribution.
Under this assumption, generalization error bounds, shedding light on the learnability of models or conducting in the design of advanced algorithms \citep{boser92}, have been proposed.} 
However, in many real applications, the data collected can be dependent, 
and therefore the i.i.d. assumption does not hold. 
There have been extensive discussions in the community on why and how the data are dependent \citep{dehling2002empirical,amini2015learning}.

\red{{\it Learning with interdependent data.}} Establishing generalization theories under dependent settings 
have received a surge of interest in recent years \citep{mohri2008stability,mohri2009rademacher,ralaivola2010chromatic,kuznetsov2017generalization}. A major line of research in this direction models the data dependencies by various types of mixing models, such as $\alpha$-mixing~\citep{rosenblatt1956central},
$\beta$-mixing~\citep{volkonskii1959some}, $\phi$-mixing~\citep{ibragimov1962some}, 
and $\eta$-mixing~\citep{kontorovich2007measure}, and so on.
Mixing models have been used in statistical learning theory to 
establish generalization error bounds
based on Rademacher complexity~\citep{mohri2009rademacher,mohri2010stability,kuznetsov2017generalization} or algorithmic stability~\citep{mohri2008stability,mohri2010stability,he2016stability}
via concentration results~\citep{kontorovich2008concentration} or independent block technique~\citep{yu1994rates}.
In these models, the mixing coefficients quantitatively measure the dependencies among data. \red{Another line of work, referred to as decoupling, studies the behavior of complex systems by decomposing a set of dependent random variables into sets of independent variables and a set of dependent variables with vanishing moments \citep{PenaGine99}. A random variable with vanishing moments has a property that its expected value converges to zero as the number of terms increases. This technique of decoupling has been successfully applied in many areas of mathematics, statistics, and engineering. }

\red{{\it Dependency graphs.}} Although the results based upon the mixing model \red{and decoupling with vanishing moments} are fruitful, they face difficulties in practical applications, as it is usually difficult to determine or estimate the quantitative dependencies among data points \red{(such as the mixing coefficients or the vanishing moments)} unless under some restrictive assumptions.
On the other hand, determining whether two data are dependent \mra{or exhibit a suitable dependency structure} is often much easier in practice. 
Thus in this paper, we focus on such a qualitative dependent setting. 
We use graphs as a natural tool to describe the dependencies among data and establish generalization
theory under such graph-dependence. The dependency graph model \red{we use} has been widely 
utilized in many other fields, 
in particular, in probability theory and statistics, where it is used to prove normal or Poisson approximation using Stein's approach, cumulants, and so on (see, for example, \citealt{janson1988normal, janson1990poisson}).
It is also heavily used in probabilistic combinatorics and statistical physics, 
such as Lov\'asz local lemma \citep{erdos1975problems}, 
Janson's inequality \citep{janson1988exponential}, along with many others.

{\it Rademacher complexity.}
We collect various concentration bounds under graph-dependence
and utilize them to derive 
Rademacher and stability generalization bounds for learning from dependent data. The basic tool used to establish generalization theory is concentration inequalities. 
\red{Standard concentration results for the i.i.d. case no longer apply for dependently distributed data, making the study a challenging task.}
\cite{janson2004large} extended Hoeffding's inequality to the sum of dependent random variables. 
This result bounds the probability that the summation of graph-dependent random variables deviates from its expected value, 
in terms of \red{the} fractional chromatic number of the dependency graph. 
Our first approach uses \red{a} similar idea, 
\mra{by dividing graph-dependent variables into sets of independent ones, 
we establish concentration bounds based on fractional colorings,
and generalization bounds via fractional Rademacher complexity.}

\red{{\it Algorithmic stability.}} PAC-Bayes bounds for classification with non-i.i.d.\ data have also been obtained based on fractional colorings of graphs in \cite{ralaivola2010chromatic}. 
These results also hold for specific learning settings such as ranking and learning from stationary $\beta$-mixing distributions.  \cite{ralaivola2015entropy} established new concentration inequalities for fractionally sub-additive and fractionally self-bounding functions of dependent variables. 
Though fundamental and elegant, the above generalization bounds are algorithm-independent. They consider the complexity of the hypothesis space and data distribution, but do not involve specific learning algorithms. To derive better generalization bounds, there is growing interest in developing algorithm-dependent generalization theories. This line of research heavily relies on the notion of algorithmic stability, which exhibits a key advantage, that is, they are tailored to specific learning algorithms, exploiting their particular properties.
Our second approach utilizes algorithmic stability to establish generalization bounds.
Note that even under the i.i.d. assumption, Hoeffding-type concentration inequalities, which bound the deviation of sample average from expectation, are not strong enough to prove stability-based generalization. On the contrary, McDiarmid's inequality characterizes the concentration of general Lipschitz functions of i.i.d.\ random variables, hence is used as the key tool for proving the stability bounds. Therefore, to build algorithmic stability theory for non-i.i.d. samples, 
we start with McDiarmid-type concentration bounds for graph-dependent random variables.

Table \ref{tab:resum} lists some generalization results
using Rademacher complexity and algorithmic stability for 
i.i.d., mixing, and graph-dependent settings, respectively.

\begin{table}[t!]
\begin{center}
\scalebox{.9}{
\begin{tabular}{ c|c|c }
\cline{2-3}
\multicolumn{1}{c}{} & Rademacher bounds & Stability bounds \\ \hline
i.i.d. 
& \cite{BartlettM02} & \cite{bousquet2002stability} 
\\
mixing conditions
& \cite{mohri2009rademacher} & \cite{mohri2008stability} \\
graph-dependence
& Theorem \ref{thm:GenBounds} (\citealt{amini2015learning})
& Theorem \ref{stabBounds} (\citealt{zhang2019mcdiarmid}) \\ \hline
\end{tabular}}
\end{center}
\caption{Rademacher complexity and stability generalization bounds for i.i.d., mixing, and graph-dependent setting.}
\label{tab:resum}
\end{table}




\red{{\it Paper organization.} \redd{In this survey, we 
begin with introducing different McDiarmid-type concentration inequalities for functions of graph-dependent random variables. 
Then we utilize these concentration bounds 
to provide upper bounds on generalization error for learning from graph-dependent data using Rademacher complexity and algorithm stability. 
}
In the reminder, Section \ref{sec:2} introduces notation and the framework. Section \ref{sec:3} establishes fractional Rademacher complexity and algorithmic stability bounds.
Section \ref{sec:4} shows how the presented framework can be utilized to derive generalization bounds 
for learning from graph-dependent data 
in a variety of practical scenarios, including learning-to-rank, multi-class classification problems, and learning from $m$-dependent data. We finally conclude this work in Section \ref{sec:5} and provide some perspective and future work.}
\section{Notation and framework}
\label{sec:2}
Throughout this paper, 
for all positive integer $n$, 
let $[n]$ denote the integer set
$\{1,2, \ldots, n\}$.
Given two integers $i < j$, let $[i, j]$ denote the integer set $\{i, i+1, \ldots, j-1, j\}$.
Let $\Omega_i$ be a Polish space for every $i\in [n]$,  
$\Om = \prod_{i\in[n]} \Omega_i = \Omega_1 \times \ldots \times \Omega_n $ be the product space, 
$\R$ be the set of real numbers, 
and $\R_+$ be the set of non-negative real numbers.
Let $\|\cdot\|_p$ denote the standard $\ell_p$-norm of a vector.
We use uppercase letters for random variables, lowercase letters for their realizations,
and bold letters for vectors.

\subsection{Graph-theoretic \redd{notation}}

We use the standard graph-theoretic notation. 
All graphs considered are finite, undirected, and simple
(no loops or multiple edges).
A graph $G = (V, E)$ consists of a set of vertices $V$, 
some of which are connected by edges in $E$.
Given a graph $G$, 
let $V(G)$ be the vertex set and $E(G)$ be the edge set.
The edge connecting a pair of \red{distinct} vertices $u, v$ is denoted by $\{ u, v \}$,
which is assumed to be unordered.
\red{The number of edges incident on a vertex is the degree of the vertex;
and we use $\Delta(G)$ to denote the maximum degree of graph $G$.}

\subsubsection{Graph covering and partitioning}
\label{sec:graphcovering}

Formally, given a graph $G$, we introduce the following definitions.

\begin{enumerate}[label=(a\arabic*)]
\item A family $\{ S_{k} \}_{k}$ of subsets of $V(G)$ is a \textit{vertex cover} of $G$ if $\bigcup S_{k} = V(G)$.

\item A vertex cover $\{ S_{k} \}_{k}$ of $G$ is a \textit{vertex partition} of $G$ 
if every vertex of $G$ is in exactly one element of $\{ S_{k} \}_{k}$.

\item A family $\{ ( S_{k}, w_{k} ) \}_{k}$ of pairs $( S_{k}, w_{k} )$, where $S_{k} \subseteq V(G)$
and $w_{k} \in [0, 1]$ is a \textit{fractional vertex cover} of $G$ 
if $\{ S_{k} \}_{k}$ is a vertex cover of $G$,
and $\sum_{k: v \in S_{k}} w_{k} = 1$ for every $v \in V(G)$.

\item \red{An independent set of $G$ is
a set of vertices of $G$, no two of which are adjacent in $G$.
Let $\mathcal I(G)$ denote the set of all independent sets of graph $G$.}

\item \label{fracCover} A fractional independent vertex cover $\{ ( I_{k}, w_{k} ) \}_{k}$ of $G$ is 
a fractional vertex cover
such that $I_{k} \in \mathcal I(G)$ for every $k$.


\item A fractional coloring of a graph $G$ is a mapping $g$ from $\mathcal I(G)$ to $[0, 1]$ 
such that $\sum_{I \in \mathcal I(G): v \in I} g(I) \ge 1$ for every vertex $v \in V(G)$. 
The fractional chromatic number $\chi_f(G)$ of $G$ is the minimum of the value 
$\sum_{I \in \mathcal I(G)} g(I)$ over fractional colorings of $G$.
See Figure \ref{fracColor} for an example.

Note that the fractional chromatic number $\chi_f(G)$ of graph $G$
is the minimum of $\sum_{k} w_{k}$ over all fractional independent vertex covers \mra{$\{ ( I_{k}, w_{k} ) \}_{k}$ of $G$ }
(see, for example, \citealt{janson2004large}).

\begin{figure}[htb]
\begin{center}
\scalebox{1.2}{
\begin{tikzpicture}
\centering
\tikzstyle{v}=[draw,circle]
\node [draw, semicircle, pattern=dots,
                inner sep=0pt, outer sep=0pt, minimum size=3mm,
                anchor=north, shape border rotate=180
                ] (i1) at (18:1.3cm) {};

\node [draw, semicircle, pattern=vertical lines, 
                inner sep=0pt, outer sep=0pt, minimum size=3mm,
                anchor=south,
                ] (i2) at (18:1.3cm) {};

\node [draw, semicircle, pattern=horizontal lines,
                inner sep=0pt, outer sep=0pt, minimum size=3mm,
                anchor=north, shape border rotate=180
                ] (j1) at (90:1.3cm) {};

\node [draw, semicircle, pattern=north west lines, 
                inner sep=0pt, outer sep=0pt, minimum size=3mm,
                anchor=south,
                ] (j2) at (90:1.3cm) {};

\node [draw, semicircle, pattern=vertical lines, 
                inner sep=0pt, outer sep=0pt, minimum size=3mm,
                anchor=north, , shape border rotate=180
                ] (k1) at (162:1.3cm) {};
                
\node [draw, semicircle, pattern=crosshatch, 
                inner sep=0pt, outer sep=0pt, minimum size=3mm,
                anchor=south,
                ] (k2) at (162:1.3cm) {};  

\node [draw, semicircle, pattern=dots,
                inner sep=0pt, outer sep=0pt, minimum size=3mm,
                anchor=north, shape border rotate=180
                ] (l1) at (234:1.3cm) {};
                
\node [draw, semicircle, pattern=north west lines, 
                inner sep=0pt, outer sep=0pt, minimum size=3mm,
                anchor=south,
                ] (l2) at (234:1.3cm) {};
                
\node [draw, semicircle, pattern=horizontal lines,
                inner sep=0pt, outer sep=0pt, minimum size=3mm,
                anchor=north, shape border rotate=180
                ] (m1) at (306:1.3cm) {};
                
\node [draw, semicircle, pattern=crosshatch, 
                inner sep=0pt, outer sep=0pt, minimum size=3mm,
                anchor=south,
                ] (m2) at (306:1.3cm) {};         
\draw
(i1) edge (m2)
;
\draw
(i2) edge (j1)
(j1) edge (k2)
(k1) edge (l2)
(l1) edge (m1);
\end{tikzpicture}
}
\end{center}
\medskip
\caption{A fractional coloring of a cycle graph $C_{5}$ of length $5$
with patterns indicating different colors.
The set of pairs 
$ \{ ( \{ i, (i+3)(\mathrm{mod}\ 5) \}, 1/2 ) \}_{1 \le i \le 5}$ is a fractional vertex cover
with the fractional chromatic number $5/2$. }
\label{fracColor}
\end{figure}

\item 
Let $H$ be a graph 
and $\{ H_x \subseteq V (G) \}_{ x \in V (H) }$ be a set of subsets of $V(G)$ indexed by the vertices of $H$. 
Each set $H_x$ is called a `bag'. The pair $(H, \{ H_x \}_{x \in V (H)})$ is an $H$-partition of $G$ if:

\begin{enumerate}[leftmargin=1cm, label=(\roman*)]
\item $\{ H_x \}_{x \in V (H)}$ is a vertex partition of $G$.
\item Distinct $u$ and $v$ are adjacent in $H$ 
if and only if there is an edge of $G$ with one endpoint in $H_u$ and the other endpoint in $H_v$.
\end{enumerate}

In graph theory, a \textit{vertex identification} (also called vertex contraction) 
is to contract a pair of vertices $u$ and $v$ of a graph 
and produces a graph in which the two vertices $u$ and $v$ are replaced with a single vertex $t$ 
such that $t$ is adjacent to the union of the vertices to which $u$ and $v$ were originally adjacent. 
Note that in vertex contraction, it does not matter if \redd{$u$ and $v$} are connected by an edge; 
if they are, the edge is simply removed upon contraction, 
this special case of vertex identification called \textit{edge contraction}.

Informally speaking, an $H$-partition of graph $G$ 
is obtained from a proper partition of $V(G)$ 
by identifying the vertices in each part, deleting loops, and replacing parallel edges with a single edge. 
$H$ is also called the \textit{quotient graph} of the graph $G$.
For brevity, we say $H$ is a partition of $G$.
For more about partitions of graphs, see, for example, \citealt{wood2009tree}.


\item A tree is a connected, acyclic graph, and a forest is a disjoint union of trees.
For a given forest $F$, 
we denote the set of (vertex sets of) disjoint trees in forest $F$ as $\T(F)$.

\item If forest $F$ is a partition of graph $G$, 
then the pair $(F, \{ F_x \subseteq V(G) \}_{x \in V (F)})$ is a \textit{tree-partition} of $G$.
The set of all tree-partitions of graph $G$ is denoted by $\textsf{TP}(G)$.
See Figure \ref{fig:tp} for an example.

Tree-partitions were
independently introduced by \cite{seese1985tree} and \cite{halin1991tree}, 
and have since been widely investigated \citep{wood2009tree}.
Essentially, 
a tree-partition of a graph is a proper partition of its vertex set into `bags', 
such that identifying the vertices in each bag produces a forest.

\begin{figure}[htb]
\centering
\scalebox{.9}{
\begin{tikzpicture}[scale=.6]
\node[circle,draw=black,fill=white] (n6) at (1,6) {$1$};
\node[circle,draw=black,fill=white] (n4) at (3,8) {$2$};
\node[circle,draw=black,fill=white] (n5) at (7,7.5)  {$4$};
\node[circle,draw=black,fill=white] (n1) at (10,8) {$6$};
\node[circle,draw=black,fill=white] (n2) at (9,6)  {$5$};
\node[circle,draw=black,fill=white] (n3) at (4,5)  {$3$};
\foreach \from/\to in {n6/n4,n4/n5,n5/n1,n1/n2,n2/n5,n2/n3,n3/n4} \draw (\from) -- (\to);

%

\node (0) at (6, 4) {$G$};

\begin{pgfonlayer}{background}
\draw[gray, line cap=round, line join=round, line width=25pt] 
(1,6) edge (3,8);
\draw[gray, line cap=round, line join=round, line width=25pt] 
(4,5) edge (7,7.5);
\draw[gray, line cap=round, line join=round, line width=25pt] 
(9,6) edge (10,8);
\end{pgfonlayer}

\node (1) at (17, 4) {$H$};
\node[circle,draw=black] (A) at (13, 6.5) {$h_1$};
\node[circle,draw=black] (C) at (16, 6.5) {$h_2$};
\node[circle,draw=black] (E) at (19, 6.5) {$h_3$};
\draw[-]  
(A) edge (C) (C) edge (E);
\end{tikzpicture}
}
\caption{A tree-partition of graph $G$ 
is $(H, \{ \{ 1, 2 \}, \{ 3, 4 \}, \{ 5, 6 \} \} )$,
where $H$ is a path on vertices $\{ h_1, h_{2}, h_{3} \}$, 
which correspond to
vertex sets $\{ 1, 2 \}, \{ 3, 4 \}$, and  $\{ 5, 6 \}$ respectively.}
\label{fig:tp}
\end{figure}

\end{enumerate}

\subsection{Probabilistic tools}

Concentration inequalities are fundamental tools in statistical learning theory.
They bound the deviation of a function of random variables from some value that
is usually the expectation.
Among the most powerful ones is McDiarmid's inequality \citep{mcdiarmid1989method}, 
which establishes sharp concentration
for multivariate functions that do not depend too much on any individual coordinate,
specifically, 
when the function satisfies $\c$-Lipschitz condition 
for a weighted hamming distance
(bounded differences condition).

Let $\I{A}$ denote the indicator function for any event $A$,
that is, $\I{A}=1$ if $A$ occurs, otherwise, $\I{A} = 0$.
We first introduce the definition of a Lipschitz function.

\begin{df}[$\c$-Lipschitz]
Given a vector $\c= (c_1, \ldots, c_n) \in \R_+^n$, 
a function $f:\Om \rightarrow \R$ is $\c$-Lipschitz 
if for all $\x = (x_1, \ldots, x_n)$ and $\x' = (x'_1, \ldots, x'_n)\in \Om$, we have
\begin{align}
\mid f(\x) - f(\x') \mid
\le \sum_{i = 1}^n c_i \I{ x_i \ne x'_i },
\label{lip}
\end{align}
where $c_i$ is the $i$-th Lipschitz coefficient of $f$
(with respect to the Hamming metric).
\end{df}

McDiarmid's inequality is based on the following bound on the moment-generating function.

\begin{lm}[\citealt{mcdiarmid1989method}]
\label{mcdiarmidlemma}
Let $\X = (X_1, \ldots, X_n)$ be a vector of independent random variables taking values in $\Om$ 
and $f:\Om\rightarrow\mathbb{R}$ be $\c$-Lipschitz.
Then for any $s>0$,
\begin{equation*}
\E\left[ e^{ s( f(\X) - \E{ f(\X)) } } \right]
\le \Exp{ \dfrac{s^{2}}{8 } \| \c \| ^2_2 }.
\end{equation*}
\end{lm}

\red{We can now state the following McDiarmid's inequality, 
which constitutes one of the pillars of our results.
It states that a Lipschitz function of independent random variables concentrates around its expectation.}
\begin{tm}[McDiarmid's inequality \citealt{mcdiarmid1989method}]
Let $f: \Om \rightarrow \R$ be $\c$-Lipschitz 
and $\X = (X_1, \ldots, X_n)$ be a vector of independent random variables 
that takes values in $\Om$.
Then for every $t>0$,
\begin{align}
\P( f(\X) - \E { f(\X) } \ge t )
\le \exp \left( - \frac{2t^2}{ \| \c \| ^2_2} \right).
\label{mc-ie}
\end{align}
\label{McDiarmid's inequality}
\end{tm}

\mra{In the following, we} extend McDiarmid's inequality to the graph-dependent case,
where the dependencies among random variables are characterized by a dependency graph.
We first define the notion of dependency graphs, which is a widely used model in probability, statistics, and combinatorics, see \cite{erdos1975problems,janson1988exponential,chen1978two,baldi1989normal} for 
some classical results.

Given a graph $G = (V, E)$,
we say that random variables $\{ X_i \}_{i \in V}$ \red{are} $G$-\textit{dependent}
if for any disjoint $S, T \subset V$ such that $S$ and $T$ are non-adjacent in $G$
(that is, no edge in $E$ has one endpoint in $S$ and the other in $T$),
random variables $\{ X_i \}_{i \in S}$ and $\{ X_j \}_{j \in T}$ are independent.
See Figure \ref{fig:dep} for an example.
Formally, we define the dependency graphs in the following.

\begin{df}[Dependency graphs]
An undirected graph $G$ is called a dependency graph of a random vector $\X=(X_1,\ldots,X_n)$ if 
\begin{enumerate}[label=(b\arabic*)]
\item $V(G)=[n]$.
\item For all disjoint $I, J \subset [n]$, if $I, J$ are not adjacent in $G$, then 
$\{X_i\}_{i \in I}$ and $\{X_j\}_{j \in J}$ are independent.
\end{enumerate}
\end{df}

\begin{figure}[tbh]
\begin{center}
\begin{tikzpicture}
[scale=.6,auto=left,every node/.style={circle,draw=black,fill=white}]
\node[circle,draw=black] (n6) at (1,6) {$1$};
\node[circle,draw=black] (n4) at (3,8) {$2$};
\node[circle,draw=black] (n5) at (7,7.5)  {$4$};
\node[circle,draw=black] (n1) at (10,8) {$6$};
\node[circle,draw=black] (n2) at (9,6)  {$5$};
\node[circle,draw=black] (n3) at (4,5)  {$3$};
\foreach \from/\to in {n6/n4,n4/n5,n5/n1,n1/n2,n2/n5,n2/n3,n3/n4} \draw (\from) -- (\to);
\end{tikzpicture}
\caption{A dependency graph $G$ for random variables $\{ X_i \}_{i \in [6]}$.
Random variables $\{  X_1, X_2 \}$ and $\{ X_5, X_6 \}$ are independent,
since disjoint vertex sets $\{ 1, 2 \}$ and $\{ 5, 6 \}$ are not adjacent in $G$.}
\label{fig:dep}
\end{center}
\end{figure}

The above \red{definition of dependency graphs} is a strong version; there are ones 
with weaker assumptions, such as the one used in Lov\'asz local lemma.
Let $K_n$ denote the complete graph on $[n]$,
that is, every two vertices are adjacent.
Then $K_n$ is a dependency graph 
for any set of variables $\{ X_i \}_{i \in [n]}$.
Note that the dependency graph for a set of random variables 
may not be necessarily unique,
and the sparser ones are the more interesting ones.

Here we introduce a widely-studied random process that generates dependent data 
whose dependency graph can be naturally constructed for illustration purposes.
Consider a data-generating procedure modeled by the spatial Poisson point process, 
which is a Poisson point process on $\R^2$,
see \cite{linderman2014discovering,kirichenko2015optimality} for discussions of using this process to 
\red{model data collections} in various machine learning applications.
The number of points in each finite region follows a Poisson distribution, 
and the number of points in disjoint regions are independent.
Given a finite set $\{ U_i \}_{i=1}^n$ of regions in $\R^2$, 
let $X_i$ be the number of points in region $U_i$ for every $i \in [n]$. 
Then the graph $G\left( [n], \{ \{ i, j \}: U_i \cap U_j \ne \emptyset \} \right)$ 
is a dependency graph of the random variables $\{ X_i \}_{i=1}^n$.


\red{
An important property of the dependency graph, 
in view of the definition of
fractional independent vertex covers,
is that if we have a fractional independent vertex cover $\{ ( I_{k}, w_{k} ) \}_{k \in [K]}$ of $G$, 
then we may decompose the sum of interdependent variables 
into a weighted sum of sums of independent variables.

\begin{lm}{\cite[Lemma 3.1]{janson2004large}}
Let $G$ be a graph,
and $\{ ( I_{k}, w_{k} ) \}_{k \in [K]}$
be a fractional independent vertex cover of $G$.
Let $\{ u_{i} \}_{i \in V(G)}$ be a set of any numbers.
Then
\begin{equation}
\label{decomposition}
\sum_{i \in V(G)} u_i
= \sum_{i \in V(G)} \sum_{k=1}^K w_k \I{ i \in I_k } u_i
= \sum_{k=1}^K w_k \sum_{i\in I_k} u_i,
\end{equation}
where each $I_{k} \in \mathcal I(G)$ is an independent set.
In particular,
we have the following.
\begin{itemize}
\item By setting $u_{i} = 1$ for each $i \in V(G)$, we have
\begin{align}
\mid V(G) \mid\,= \sum_{k=1}^K w_k \mid I_k \mid.
\label{decom-size}
\end{align}
\item By letting $\{ u_{i} \}_{i \in V(G)}$ 
be some $G$-dependent variables $\{ X_{i} \}_{i \in V(G)}$,
we have \eqref{decomposition} becomes 
a weighted sum of independent random variables
$\{ X_i \}_{i\in I_k}$.
\end{itemize}
\end{lm}
}

\subsection{Concentration bounds for decomposable functions}

Notice that McDiarmid's inequality applies to independent random variables. 
\cite{janson2004large} derived a Hoeffding-like inequality for graph-dependent random variables
by decomposing the sum into sums of independent variables.
Janson's bound is a special case of McDiarmid-type inequality tailored for interdependent random variables, especially when the function involves summation.

\begin{tm}[Janson's concentration inequality,~\citealt{janson2004large}]
Let random vector $\X$ be $G$-dependent such that for every $i \in V(G)$, 
random variable $X_i$ takes values in a real interval of length $c_i \ge 0$.
Then, for every $t > 0$, 
\begin{align}
\P\left( \sum_{i \in V(G)} X_i - \E { \sum_{i \in V(G)} X_i } \ge t \right)
\le \Exp{ - \frac{2t^2}{ \chi_f(G) \| \c \| ^2_2} },
\label{eqn:janson}
\end{align}
where $\c= (c_i )_{i \in V(G)}$ and $\chi_f(G)$ is 
the fractional chromatic number of $G$.
\label{Janson inequality}
\end{tm}

\red{
We will extend this result,
and obtain similar concentration results
under certain decomposability constraints 
for Lipschitz functions of graph-dependent random variables defined in Definition \ref{def:decLip}.}

\red{
\begin{df}[Decomposable $\c$-Lipschitz functions]
\label{def:decLip}
Given a graph $G$ on $n$ vertices and a vector $\c= (c_i)_{i \in [n]} \in \R_+^n$,
a function $f:\Om \rightarrow \R$ is {\it decomposable $\c$-Lipschitz}
with respect to graph $G$ if for all $\x = (x_1, \ldots, x_n) \in \Om$
and for all fractional independent vertex covers $\{ ( I_{j}, w_{j} ) \}_{j}$ of $G$,
there exist $(c_{i})_{i \in I_{j}}$-Lipschitz functions $\{ f_{j}: \Om_{I_{j}} \rightarrow \R \}_{j}$ such that
\begin{align}
f( \x ) = \sum_{ j } w_j f_{j}( \x_{I_j} ),
\label{decom}
\end{align}
where
for every set $V \subseteq [n]$,
we write
$\Om_V := \prod_{i \in V} \Omega_i$,
and $\x_V := \{ X_i \}_{i \in V}$.
\end{df}
}

\begin{tm}[\citealt{usunier2005generalization, amini2015learning}]
\label{thm:FracConBounds}

Let function $f: \Om\rightarrow \R$ be decomposable $\c$-Lipschitz,
and $\Om$-valued random vector $\X$ be $G$-dependent.
Then for $t > 0$,
\begin{align}
\P\left( f(\X) - \E { f(\X) } \ge t \right)
\le \Exp{ -\frac{2t^2}{\chi_f(G) \| \c \| ^2_2 } }.
\label{colorB}
\end{align}
\end{tm}

\begin{re}
\label{chromRE}
The chromatic number $\chi(G)$ of a graph $G$ is 
the smallest number of colors needed to color the vertices of $G$
such that no two adjacent vertices share the same color.
Let $\Delta(G)$ denote the maximum degree of $G$.
It is well-known that $\chi_f(G) \le \chi(G) \le \Delta(G)+1$, 
(see, for example, \cite{bollobas1998modern}). 
Thus in our bound \eqref{colorB}, 
we can substitute $\chi_f(G)$ with $\chi(G)$ or $\Delta(G)+1$,
which may be easier to estimate in practice.
\end{re}

\begin{proof}[Proof of Theorem \ref{thm:FracConBounds}]
Following 
the Cram\'er-Chernoff method
(see, for example, \citealt{boucheron2013concentration}),
we have for any $s>0$ and $t>0$,
\begin{align}
\P\left( f(\X) - \E f(\X) \ge t \right) 
\le e^{-st} \E \left[ e^{  s ( f(\X) - \E f(\X) ) } \right].
\label{mark}
\end{align}

Let $\{ ( I_{j}, w_{j} ) \}_{j \in [J]}$ be a fractional independent vertex cover of the dependency graph $G$ with 
\begin{align}
\sum_{j=1}^{J} w_j = \chi_f(G).
\label{sumw}
\end{align}
\red{Utilizing the decomposition property of the Lipschitz function \eqref{decom},}
the moment-generating function on the right-hand side of \eqref{mark} can be written as
\begin{align*}
\E \left[ e^{  s ( f(\X) - \E f(\X) ) } \right]
= \EE{ \Exp{  \sum_{j=1}^J  s w_j (  f_j(I_j) - \E f_j(I_j) ) }},
\end{align*}
\redd{where each $f_j(I_j) 
= f_j( \X_{I_j} ) $ is some Lipschitz function of independent variables
$\{ X_i \}_{I \in I_j}$.}

\mra{Now, let $\{p_1, \ldots, p_J\}$ be any set of $J$ strictly positive reals that sum to 1. 
Since $\sum_{j=1}^J \omega_j/\chi_f(G)=1$ \red{by \eqref{sumw}}, 
using the convexity of the exponential function and Jensen's inequality, 
we obtain that} 
\begin{align}
\E \left[ e^{  s ( f(\X) - \E f(\X) ) } \right]
&= \EE{ \Exp{  \sum_{j = 1}^J p_j \dfrac{sw_j}{p_j}  (  f_j(I_j) - \E f_j(I_j) ) } } \notag\\
&\le \EE{ \sum_{j = 1}^J p_j \Exp{  \dfrac{sw_j}{p_j}  (  f_j(I_j) - \E f_j(I_j) ) } } \notag\\
&= \sum_{j = 1}^J p_j \EE{ \Exp{  \dfrac{sw_j}{p_j} (f_j(I_j) - \E f_j(I_j) ) } },
\label{sumInd}
\end{align}
\red{where the last step is by the linearity of expectation.}
\redd{Note that each subset $I_j$ in 
summation \eqref{sumInd} 
is an independent set,
and therefore corresponds to independent variables.}
Hence applying Lemma \ref{mcdiarmidlemma} to each expectation that appears in the above summation gives
\begin{align*}
\sum_{j = 1}^J p_j \EE{ \Exp{  \dfrac{sw_j}{p_j} (f_j(I_j) - \E f_j(I_j) ) } }
\le \sum_{j = 1}^J p_j \Exp{ \dfrac{s^2 w_j^2}{8p_j^2}  \sum_{i \in I_j} c_i ^2 }.
\end{align*}
\mra{By rearranging terms in the exponential of the right-hand side of the inequality above 
and setting 
\begin{align*}
p_j = \dfrac{ w_j \sqrt{  \sum_{i \in I_j} c_i^2 } }
{ \sum_{j = 1}^J \left( w_j \sqrt{ \sum_{i \in I_j} c_i^2 } \right) },
\end{align*}
we have that}
\begin{align*}
\sum_{j = 1}^J p_j \Exp{ \dfrac{s^2 w_j^2}{8p_j^2} \sum_{i \in I_j} c_i^2 }
&= \sum_{j = 1}^J p_j \Exp{ 
\dfrac{s^2}{8} \left(\sum_{j = 1}^J w_j \sqrt{ \sum_{i \in I_j} c_i^2 }\right)^2 } \\
&= \Exp{ 
\dfrac{s^2}{8} \left(\sum_{j = 1}^J w_j \sqrt{ \sum_{i \in I_j} c_i^2 }\right)^2 },
\end{align*}
\red{where the last equality is 
by recalling that the sum of $p_i$ equals $1$.}
By Cauchy-Schwarz inequality,
\begin{align*}
\left(\sum_{j = 1}^J w_j \sqrt{ \sum_{i \in I_j} c_i^2 }\right)^2
&= \left(\sum_{j = 1}^J \sqrt{w_j } \sqrt{w_j \sum_{i \in I_j} c_i^2 }\right)^2 \\
&\le \left(\sum_{j = 1}^J w_j \right) \left(\sum_{j = 1}^J w_j \sum_{i \in I_j} c_i^2 \right)
= \chi_f(G) \sum_{i \in V(G)} c_i^2,
\end{align*}
where the last equality is due to decomposition \eqref{decomposition}
\red{and equation \eqref{sumw}.
The proof is then 
completed by choosing 
$s = 4t/ (\chi_f(G) \sum_{i \in V(G)} c_i^2)$ in \eqref{mark}.
}
\end{proof}

\subsection{Concentration bounds for general Lipschitz functions}

\redd{
We have demonstrated concentration results for functions with specific decomposable constraints. Moving forward, we extend our study to encompass more general Lipschitz functions.  
To begin with, we present concentration results for scenarios involving forest-dependence, wherein the dependency graphs are structured as forests.  It is worth recalling that a forest is a disjoint union of trees.
}

\begin{tm}[\citealt{zhang2019mcdiarmid, zhang}]

Let function $f: \Om\rightarrow \R$ be $\c$-Lipschitz,
and $\Om$-valued random vector $\X$ be $G$-dependent.
If $G$ is a disjoint union of trees $\{ T_i \}_{i\in[k]}$.
Then for $t > 0$,
\begin{align}
\P\left( f(\X) - \E { f(\X) } \ge t \right)
&\le \Exp{ - \frac{2t^2}{ \sum^k_{i=1} c_{\min,i}^2 + \sum_{ \{i, j\} \in E(G) } ( c_i + c_j)^2 } },
\label{eqn:Dependency forest}
\end{align}
where $c_{\min,i} := \min \{ c_j : j \in V(T_i) \}$ for all $i \in [k]$.
\label{ie-dep-f}
\end{tm}

\red{
The proof of this theorem
is by 
first properly ordering $\{ X_{i} \}_{i \in V(G)}$
as $( X_{i} )_{i \in [n]}$,
and rewriting $f(\X) - \E f(\X)$ as a summation 
$\sum_{i \in [n]} V_i$, where 
\begin{align*}
V_i := \E [ f(\X) \mid X_1, \ldots X_i ]
- \E [ f(\X) \mid X_1, \ldots X_{i-1} ].
\end{align*}
In the proof, 
each tree $T_{i}$ is rooted 
by choosing the vertex 
with the minimum Lipschitz coefficient $\min \{ c_j : j \in V(T_i) \}$
in that tree as the root.
It can be shown that 
for some suitable ordering,
\redd{each 
$V_i$ ranges in an interval of length at most $c_i+c_j$, 
where $j$ is the parent of $i$ in the tree,
or simply $c_i$ (if $i$ corresponds to a root vertex).}
The theorem then follows
by applying the Chernoff-Cram\'er technique to $\sum_{i = 1}^n V_i$. 
The detailed proof is a bit involved 
and can be found in \cite{zhang}.
}

\begin{re}
If random variables $(X_1, \ldots, X_n)$ are independent, 
then the empty graph $\overline{K}_n = ([n], \emptyset)$ is a valid dependency graphs for $\{ X_i \}_{i\in[n]}$. 
In this case, inequality (\ref{eqn:Dependency forest})
gets reduced to the McDiarmid's inequality \eqref{mc-ie},
since each vertex is treated as a tree.

If all Lipschitz coefficients are of the same value $c$, then 
the denominator of the exponent
in (\ref{eqn:Dependency forest}) becomes $ k c^2 + 4 (n-k) c^2 = (4n - 3k) c^2 $,
since the number of edges in the forest is $n - k$.
The denominator in Janson's bound (\ref{eqn:janson}) is $2nc^2$,
since the fractional chromatic number of any tree is $2$.
Thus if $k \ge 2n/3$, then bound \eqref{eqn:Dependency forest} is \mra{tighter} than \mra{Janson's concentration inequality \eqref{eqn:janson}.}
\label{re}
\end{re}

\subsubsection{Concentration for general graphs}

In this subsection, we consider the concentration of general Lipschitz functions
\redd{
of variables whose dependency graph
may not be a forest}.
This is
by utilizing tree-partitions of 
the dependency graphs via vertex identifications,
and then applying the forest-dependent results obtained.

\begin{tm}
Let function $f: \mathbf{\Omega}\rightarrow \R$ be $\c$-Lipschitz,
and $\mathbf{\Omega}$-valued random vector $\X$ be $G$-dependent.
Then for any $t>0$,
\begin{align*}
\P( f(\X) - \E f(\X) \ge t )
\le \Exp { - \frac{2t^2}{ D(G, \c) } },
\end{align*}
where 
\begin{align*}
D(G, \c)
:= \min_{(F, \{ F_x \}_{x \in V (F)}) \in \textsf{TP}(G)} 
\left( \sum_{T \in \T(F)} \wc_{\min, T}^2 
+ \sum_{\{ u, v \} \in E(F)} ( \wc_u + \wc_v)^2  \right),
\end{align*}
with $\wc_u := \sum_{i\in F_u} c_i$ for all $u \in V(F)$
and $\wc_{\min, T} := \min \{ \wc_i : i \in V(T) \}$ for all $T \in \mathcal T(F)$.
\label{Concentration Inequality for Dependency Graph}
\end{tm}

\begin{proof}
For every $u \in V(F)$, 
we define a random vector $\mathbf{Y}_u= \{ X_i \}_{ i \in F_u}$,
and treat each $\mathbf{Y}_u$ as a random variable. 
We then define a new random vector $\mathbf{Y}=(\mathbf{Y}_u)_{u\in V(F)}$, 
and let $g(\mathbf{Y})=f(\X)$. 
It is easy to check that $g$ is $\widetilde{\c}$-Lipschitz
\red{by the triangle inequality,}
where $\widetilde{\c}=(\wc_u)_{u\in V(F)}$. 
Hence the theorem immediately follows from Theorem~\ref{ie-dep-f}.
\end{proof}

It is useful to define the notion of {\it forest complexity}, which depends only on the graph,
especially when the Lipschitz coefficients are of the same order.

\begin{df}[Forest complexity]\label{lambda}
The forest complexity of a graph $G$ is defined by
\begin{align*}
\Lambda(G) := \min_{ (F, \{ F_x \}_{x \in V (F)}) \in \textsf{TP}(G) } 
\left( \sum_{T \in \T(F)} \min_{ u \in T } \mid  F_u \mid ^2 
+ \sum_{\{ u, v \} \in E(F)} \mid  F_u \cup F_v\mid ^2 \right),
\end{align*}
where the minimization is over all tree-partitions of $G$.
\end{df}

\begin{re}
The width of a tree-partition is the maximum number of vertices in a bag. 
The tree-partition-width $\tpw(G)$ of $G$ is the minimum width of a tree-partition of $G$.
Let $F \in \textsf{TP}(G)$ be the tree-partition 
with tree-partition width $\tpw(G)$. Then
\begin{align*}
\Lambda(G)
\le \mid\T(F)\mid  \tpw(G)^2 + 4\mid  E(F) \mid \tpw(G)^2
= ( \mid V(F)\mid  + 3\mid  E(F) \mid  ) \tpw(G)^2,
\end{align*}
since the number of disjoint trees in a forest $F$ equals $\mid V(F)\mid  - \mid  E(F) \mid $.
Upper bounds on tree-partition-width $\Lambda(G)$
can be obtained using treewidth and the maximum degree of $G$,
and are beyond the scope of this paper,
see \cite{wood2009tree} for more details.
\end{re}

\red{
If all the Lipschitz coefficients are of the same value,
then Theorem \ref{Concentration Inequality for Dependency Graph} gets simplified.
}
\begin{co}
Let function $f: \Om\rightarrow \R$ be Lipschitz 
with the same coefficient $c$, 
and $\Om$-valued random vector $\X$ be $G$-dependent.
Then for $t > 0$,
\begin{align}
\P( f(\X) - \E f(\X) \ge t )
\le \Exp{ - \frac{2t^2}{ \Lambda(G) c^2 }}.
\end{align}
\label{ie-dep-graph}
\end{co}

%
%
%
%
%
%
%
%

Similar to the theorems derived above, 
Corollary \ref{ie-dep-graph} also gives an exponentially decaying bound on the probability of deviation. 
The rate of decay is determined by the Lipschitz coefficients of the function, 
and the forest complexity of the dependency graph. 
Intuitively, the closer the dependency graph is to a forest, 
the faster the deviation probability decays. 
This uncovers how the dependencies among random variables influence concentration.


\subsubsection{Examples}
\label{Examples and Applications}

\redd{Here we present several explicit examples 
to demonstrate and estimate the forest complexity where random variables are structured as graphs.
All these examples naturally emerge in the context of random processes that are intricately intertwined within graph structures.
}

\begin{ex}[$G$ is a tree]
In this case, 
$\Lambda(G) \le \mid E(G)\mid (1+1)^2 + 1 = 4n - 3$. 
We get an upper bound of $\Lambda(G)$ that is linear in the number of variables, 
which is comparable to 
\mra{Janson's concentration inequality} up to some constant factor (see (\ref{eqn:janson}) with $\chi_f(G)=2$ and Remark \ref{re}).
\label{example tree}
\end{ex}

\begin{ex}[$G$ is a cycle $C_n$]
If $n$ is even, a tree-partition is illustrated in Figure~\ref{evencycle}, 
where the resulting forest is a path $F$ of length $n/2$
with each gray belt representing a `bag'. 
We will keep this convention for the rest of this paper. 

By the illustrated tree-partition, 
$\Lambda(G) \le 2 \times (1+2)^2 + (n/2-2)(2+2)^2+1 = O(n) $.
When $n$ is odd, according to the tree-partition shown in Figure~\ref{oddcycle},
$\Lambda(G)\le (1+2)^2 + (\frac{n - 1}{2}-1)(2+2)^2 +1 = O(n)$. 
Since $\chi_f \ge 2$ for cycles, our bound is again comparable to \mra{Janson's concentration inequality} \eqref{eqn:janson}
up to some constant multiplicative factor.
\label{example cycle}
\begin{figure}[htb]
\begin{minipage}{.5\textwidth}
\centering
\begin{subfigure}[c]{.4\textwidth}\centering
\begin{tikzpicture}
\node[diamond,draw=black,fill=white,scale=.8] (i) at (0:1cm) {};
\node[circle,draw=black,fill=white] (j) at (60:1cm) {};
\node[diamond,draw=black,fill=white,scale=.8] (k) at (120:1cm) {};
\node[rectangle,draw=black,fill=white] (l) at (180:1cm) {};
\node[regular polygon,regular polygon sides=3,draw,scale=0.6] (m) at (240:1cm) {};
\node[rectangle,draw=black,fill=white] (n) at (300:1cm) {};
\draw[-]
(i) edge (j) (j) edge (k) (k) edge (l) (l) edge (m) (m) edge (n) (n) edge (i);
\begin{pgfonlayer}{background}
\draw[edgeBIG] (0:1cm) edge (120:1cm);
\draw[edgeBIG] (180:1cm) edge (300:1cm);
\end{pgfonlayer}
\end{tikzpicture}
\caption*{$G$}
\end{subfigure}
\begin{subfigure}[c]{.2\textwidth}\centering
\begin{tikzpicture}[scale=.6]
\node[circle,draw=black,fill=white,scale=.6] (A) at (0, 3) {};
\node[diamond,draw=black,fill=white,scale=.5] (B) at (0, 2) {};
\node[rectangle,draw=black,fill=white,scale=0.7] (C) at (0, 1) {};
\node[regular polygon,regular polygon sides=3,draw,scale=0.4] (E) at (0, 0) {};
\draw[-]  
(A) edge (B) (B) edge (C) (C) edge (E);
\end{tikzpicture}
\caption*{$F$}
\end{subfigure}
\caption{A tree-partition of $C_6$.}
\label{evencycle}
\end{minipage}%
\begin{minipage}{.5\textwidth}
\centering
\begin{subfigure}[c]{.4\textwidth}\centering
\begin{tikzpicture}
\node[rectangle,draw=black,fill=white] (i) at (0:1cm) {};
\node[circle,draw=black,fill=white] (j) at (72:1cm) {};
\node[circle,draw=black,fill=white] (k) at (144:1cm) {};
\node[rectangle,draw=black,fill=white] (l) at (216:1cm) {};
\node[regular polygon,regular polygon sides=3,draw,scale=0.6] (m) at (288:1cm) {};
\draw
(i) edge (j) (j) edge (k) (k) edge (l) (l) edge (m) (m) edge (i);
\begin{pgfonlayer}{background}
\draw[edgeBIG] (0:1cm) edge (216:1cm);
\draw[edgeBIG] (72:1cm) edge (144:1cm);
\end{pgfonlayer}
\end{tikzpicture}
\caption*{$G$}
\end{subfigure}
\begin{subfigure}[c]{.2\textwidth}\centering
\begin{tikzpicture}[scale=.6]
\node[circle,draw=black,fill=white,scale=.6] (A) at (0, 3) {};
\node[rectangle,draw=black,fill=white,scale=0.7] (C) at (0, 1.5) {};
\node[regular polygon,regular polygon sides=3,draw,scale=0.4] (E) at (0, 0) {};
\draw[-]  
(A) edge (C) (C) edge (E);
\end{tikzpicture}
\caption*{$F$}
\end{subfigure}
\caption{A tree-partition of $C_5$.}
\label{oddcycle}
\end{minipage}%
\end{figure}
\end{ex}

\begin{ex}[$G$ is a grid] Suppose $G$ is a two-dimensional $(m\times m)$-grid. Then $n=m^2$. Considering the tree-partition illustrated in Figure~\ref{grid}, we have
\begin{align*}
\Lambda(G) 
\le 1 + 2 \sum_{i=1}^m (2m - 1)^2
= \dfrac{2}{3} m(2m+1)(2m-1)+1
= O(m^3)
= O(n^\frac{3}{2}).
\end{align*}
\label{example grid}
\begin{figure}[H]
\centering
\begin{subfigure}{.6\textwidth}
\begin{tikzpicture}\centering
\draw[-] (0, 0) grid [step=1] (3, 3);

\node [star,draw,scale=0.8,fill=white]  (00) at (0, 0) {};
\node [diamond,draw=black,fill=white,scale=.8]  (33) at (3, 3) {};

\node [regular polygon,regular polygon sides=5,draw,scale=0.8,fill=white]  (01) at (0, 1) {};
\node [regular polygon,regular polygon sides=5,draw,scale=0.8,fill=white]  (10) at (1, 0) {};

\node [rectangle,draw=black,fill=white]  (02) at (0, 2) {};
\node [rectangle,draw=black,fill=white]  (11) at (1, 1) {};
\node [rectangle,draw=black,fill=white]  (20) at (2, 0) {};

\node [regular polygon,regular polygon sides=3,draw,scale=0.6,fill=white]  (13) at (1, 3) {};
\node [regular polygon,regular polygon sides=3,draw,scale=0.6,fill=white]  (22) at (2, 2) {};
\node [regular polygon,regular polygon sides=3,draw,scale=0.6,fill=white]  (31) at (3, 1) {};

\node [regular polygon,regular polygon sides=6,draw,scale=0.8,fill=white]  (23) at (2, 3) {};
\node [regular polygon,regular polygon sides=6,draw,scale=0.8,fill=white]  (32) at (3, 2) {};

\node [circle,draw=black,fill=white]  (03) at (0, 3) {};
\node [circle,draw=black,fill=white]  (12) at (1, 2) {};
\node [circle,draw=black,fill=white]  (21) at (2, 1) {};
\node [circle,draw=black,fill=white]  (30) at (3, 0) {};

\node []  () at (-2.5, 1.5) {$G$};
\begin{pgfonlayer}{background}
\draw[edgeBIG] (2,3) edge (3,2);

\draw[edgeBIG] (1,3) edge (3,1);

\draw[edgeBIG] (0,3) edge (3,0);

\draw[edgeBIG] (0,2) edge (2,0);

\draw[edgeBIG] (0,1) edge (1,0);
\end{pgfonlayer}
\end{tikzpicture}
\end{subfigure}

\vspace{.5cm}

\begin{subfigure}[c]{.6\textwidth}
\begin{tikzpicture}
\centering
\node[] (0) at (-1, 0) {$F$};
\node[star,draw,scale=0.8,fill=white] (A) at (0, 0) {};
\node[regular polygon,regular polygon sides=5,draw,scale=0.8,fill=white] (B) at (1, 0) {};
\node[rectangle,draw=black,fill=white] (C) at (2, 0) {};
\node[circle,draw=black,fill=white] (D) at (3, 0) {};
\node[regular polygon,regular polygon sides=3,draw,scale=0.6,fill=white] (E) at (4, 0) {};
\node[regular polygon,regular polygon sides=6,draw,scale=0.8,fill=white] (F) at (5, 0) {};
\node[diamond,draw=black,fill=white,scale=.8] (G) at (6, 0) {};
\draw[-]  
(A) edge (B) (B) edge (C) (C) edge (D) (D) edge (E) (E) edge (F) (F) edge (G);
\end{tikzpicture}
\end{subfigure}
\caption{A tree-partition of $4\times 4$ gird.}
\label{grid}
\end{figure}
\end{ex}

\section{Generalization for learning from graph-dependent data}
\label{sec:3}


We now apply \redd{the} concentration bounds obtained above 
to derive generalization bounds 
for supervised learning from graph-dependent data.
Let
\begin{align*}
\S := ((x_1, y_1), \ldots, (x_n, y_n)) \in (\mathcal X \times \mathcal Y)^n
\end{align*}
be a $G$-dependent training sample of size $n$,
where $\mathcal X$ denotes the input space and $\mathcal Y$ denotes the 
set of labels.
Let $\D$ be the underlying distribution of data on $\mathcal X \times \mathcal Y$.
\rui{Note that the sample $\S$ 
contains dependent data with the same marginal distribution $\D$.}

\mra{Further we fix some $\ell: \mathcal Y \times \mathcal Y \rightarrow \R_+$ as a non-negative loss function. 
For any hypothesis $f: \mathcal X \rightarrow \mathcal Y$, 
the empirical error on sample $\S$ is defined by}
\[
\RH_\S(f) := \frac{1}{n} \sum_{i = 1}^n \ell(y_i, f ( x_i )). 
\]
For learning from dependent data, the generalization error can be defined in various ways. 
We adopt the following widely-used one ~\citep{meir2000nonparametric,lozano2006convergence,steinwart2009fast,hang2014fast}
\begin{equation}
R(f) := \E_{(x, y)\sim \D} [ \ell(y, f(x)) ],
\label{gen error}
\end{equation}
which assumes that the test data is independent of the training sample.

\subsection{Generalization bounds via fractional Rademacher complexity}

\mra{Our first approach is based on Rademacher complexity \citep{BartlettM02}.
This approach can be extended to accommodate interdependent data by utilizing the decomposition into independent sets
described in Section \ref{sec:graphcovering}.}

\begin{df}[Fractional Rademacher complexity, 
\citealt{usunier2005generalization}]\label{df:FracCover} 

\mra{Let $\{ ( I_{j}, w_{j} ) \}_{j}$ be a fractional independent vertex cover of a dependency graph $G$ 
constructed over a training set $\S$ of size $n$,  with $\sum_j w_j = \chi_f(G)$. 
Let $\F=\{ f:\mathcal X \rightarrow \mathcal Y\}$ be the hypothesis class. 
Then, the {\it empirical fractional Rademacher complexity} of $\F$ given $\S$ is defined by
} 
\begin{equation}
\label{eq:FracRadCom}
 \widehat{\mathfrak R}_{\S}^{\star}(\F) = \frac{1}{\mra{n}} \E_{\boldsymbol{\sigma}} 
\left[ \sum_{j} w_j 
\sup_{f \in \F} \left( \sum_{i \in I_j} \sigma_i f ( x_i ) \right) \right],
\end{equation}
\mra{where $\boldsymbol{\sigma}=(\sigma_i)_{1\le i\le n}$ 
denote a vector of $n$ independent Rademacher variables,
that is, $\mathbb{P}(\sigma_i=-1)=\mathbb{P}(\sigma_i=+1)=1/2$ for each $i \in [n]$. 
Moreover, the {\it fractional Rademacher complexity} of $\F$ is defined by
\begin{align*}
\mathfrak R^{\star}(\F)
= \E_{\S} \left[ \widehat{\mathfrak R}_{\S}^{\star}(\F) \right].
\end{align*}
}
\end{df}

\begin{re}
\label{RadComiid}
In the i.i.d. situation, 
the set of singleton vertices is a valid fractional independent vertex cover, 
and the fractional Rademacher complexity \eqref{eq:FracRadCom} 
simplifies to the original empirical Rademacher complexity \citep{BartlettM02} defined by
\begin{equation}
\label{eq:RadCom}
\widehat{\mathfrak R}_{\S}(\F) 
= \frac{1}{n} \E_{\boldsymbol{\sigma}} 
\left[ \sup_{f \in \F} \left( \sum_{i \in [n]} \sigma_i f ( x_i ) \right) \right].
\end{equation}
Additionally, because the former is a sum of empirical Rademacher complexities, 
it enables one to get estimates by extending the properties of the empirical Rademacher complexity.
\end{re}

In the following, we give an example of a function class 
of linear functions with bounded-norm weight vectors,
for which the empirical Rademacher averages
can be bounded directly.

\begin{tm}
\label{thm:FracRad}
\redd{Let $\F=\{x\mapsto \langle \mathbf{w},\phi(x)\rangle : \|\mathbf{w}\|\le B\}$ 
be a class of linear functions with bounded weights
in a feature space
such that 
$\|\phi ( x )\| \le \Gamma$ for all $x$.}
Then
\begin{equation}
    \widehat{\mathfrak R}_{\S}^{\star}(\F) \le B \Gamma\sqrt{\frac{\chi_f(G)}{n}}.
    \label{FracRad-linear}
\end{equation}
\end{tm}
\begin{proof}

\red{In view of the definition of the empirical fractional Rademacher complexity
\eqref{eq:FracRadCom},
by the linearity of expectation, we have
\begin{align*}
\widehat{\mathfrak R}_{\S}^{\star}(\F) &= \frac{1}{n} \sum_{j} w_j 
\E_{\boldsymbol{\sigma}} 
\left[  
\sup_{\|\mathbf{w}\|\le B} \left( \sum_{i \in I_j} 
\left\langle \mathbf{w}, \sigma_i \phi ( x_i ) \right\rangle\right) \right]\\
&\le  \frac{B}{n}  \sum_{j} w_j \E_{\boldsymbol{\sigma}} \norm{\sum_{i \in I_j} \sigma_i \phi ( x_i )} \le  \frac{B}{n}  \sum_{j} w_j \left(\E_{\boldsymbol{\sigma}}  \norm{\sum_{i \in I_j} \sigma_i \phi ( x_i )}^2\right)^{1/2},
\end{align*}
where the first inequality is by noting $\|\mathbf{w}\|\le B$,
and applying Cauchy-Schwarz inequality to the inner product,
and the second inequality is by Jensen's inequality.
}

As the Rademacher variables are independent, 
we have $\E[\sigma_i\sigma_k]=\E[\sigma_i]\E[\sigma_k]=0$ for any distinct $i, k$. 
Hence we have
\begin{align*}
\E_{\boldsymbol{\sigma}}  \norm{\sum_{i \in I_j} \sigma_i \phi ( x_i )}^2
= \E_{\boldsymbol{\sigma}} \left[ \sum_{i,k \in I_j} \sigma_i\sigma_k \langle \phi ( x_i ),\phi(x_k) \rangle \right]
= \sum_{i \in I_j} \|\phi ( x_i )\|^2,
\end{align*}
and therefore,
\[
 \widehat{\mathfrak R}_{\S}^{\star}(\F) \le  \frac{B}{n}  \sum_{j} w_j  \left(\sum_{i \in I_j} \|\phi ( x_i )\|^2\right)^{1/2}.
\]

Since we have $\|\phi ( x_i )\| \le \Gamma$ in the feature space,
then
\begin{align*}
\widehat{\mathfrak R}_{\S}^{\star}(\F) 
\le \frac{B\Gamma}{n}  \sum_{j} w_j  \sqrt{\mid I_j \mid}
= \frac{B\Gamma\chi_f(G)}{n}  \sum_{j} \frac{w_j}{\chi_f(G)}  \sqrt{\mid I_j \mid}.
\end{align*}
By noticing that $\sum_{j} w_j / \chi_f(G) = 1$, 
using Jensen's inequality for the square root function yields
\begin{align*}
\widehat{\mathfrak R}_{\S}^{\star}(\F) 
\le \frac{B\Gamma\chi_f(G)}{n} \sqrt{ \sum_{j} \frac{w_j}{\chi_f(G)} \mid I_j \mid }
= \frac{B\Gamma\sqrt{\chi_f(G)}}{n}\sqrt{\sum_j w_j \mid I_j \mid}.
\end{align*}
The result follows then by noting $\sum_j w_j \mid I_j \mid=n$
by \eqref{decom-size}.
\end{proof}

\begin{re}
Note that $\phi$ could be the feature mapping corresponding to 
the last hidden layer of a neural network, or a kernel function.
\red{
In particular, 
under the assumption of Theorem \ref{thm:FracRad},
let $\phi$ be a feature mapping associated to 
a kernel $K$ such that $K(x, x) \le \Gamma^{2}$
for all $x$.
Then the standard Rademacher complexity 
of kernel-based hypotheses \cite[Theorem 6.12]{mohri2018foundations} gives that 
$\widehat{\mathfrak R}_{\S}(\F) \le B\Gamma/\sqrt{n}$,
and in comparison, our bound \eqref{FracRad-linear} has an additional factor $\sqrt{\chi_{f}(G)}$,
which becomes exactly $1$ as in Remark \ref{RadComiid}.
}

It is also worth noting that the fractional Rademacher complexity is defined 
for a given fractional cover. 
In general, our analysis holds 
for any optimal fractional cover; 
nevertheless, various cover selections may result in different bound values. Nonetheless, in practice, this influence is unlikely to have a significant impact. 
\end{re}

We now obtain generalization bounds using the fractional Rademacher complexity.

\begin{tm}[\citealt{usunier2005generalization, amini2015learning}]
\label{thm:GenBounds}
Given a sample $\S$ of size $n$ with dependency graph $G$
and a loss function $\ell : \mathcal Y \times \widehat{\mathcal Y} \rightarrow [0, M]$. 
\red{Let $\F$ denote the hypothesis class.}
Then, for any $\delta \in (0, 1)$, with probability at least $1 - \delta$,
we have, \red{for all $f \in \F$, that}
\begin{equation}
\label{eq:FracRadGenBounds}
R(f) 
\le \RH_{\S}(f) 
+ 2
\mra{\mathfrak R^{\star}(\ell\circ\F)}
+ M\sqrt{\dfrac{\chi_f(G)}{2\mra{n}} \Log{\dfrac{1}{\delta}} },
\end{equation}
and
\begin{equation}
\label{eq:EmpFracRadGenBounds}
R(f) 
\le \RH_{\S}(f) 
+ 2\widehat{\mathfrak R}_\S^{\star}(\ell\circ\F)
+ 3M\sqrt{\dfrac{\chi_f(G)}{2\mra{n}} \Log{\dfrac{2}{\delta}} },
\end{equation}
where $\ell\circ\F=\{(x,y)\mapsto \ell(y,f(x)) ~\mid~  f \in \F\}$.
\end{tm}

\begin{proof}
For any $f\in \F$, we have $\RH_\S(f)$ is an unbiased estimator of $R(f)$,
\rui{since the data points in the sample $\S$ are assumed to 
be $G$-dependent and have the same marginal distribution.}
Hence considering \red{a $G$-dependent ``ghost'' sample 
$\S'
= ((x_1', y_1'), \ldots, (x_n', y_n'))$
that is independently
generated from the same distribution as $\S$}, we have
\begin{align*}
\sup_{f \in \F} (R(f)-\RH_\S(f)) 
= \sup_{f \in \F} \left(\E_{\S'}\RH_{\S'}(f) -\RH_\S(f) \right)
= \sup_{f \in \F} \left(\E_{\S'}\left[\RH_{\S'}(f) -\RH_\S(f)\right]\right).
\end{align*}
\rui{Let $\{ ( I_{j}, w_{j} ) \}_{j \in [J]}$ be a fractional independent vertex cover of the dependency graph $G$ with $\sum_j w_j = \chi_f(G)$.}
By Jensen's inequality and the convexity of the supremum, 
we get
\begin{align*}
&\sup_{f \in \F} \left(\E_{\S'}\left[\RH_{\S'}(f) -\RH_\S(f)\right]\right) 
\le \E_{\S'} \left[ \sup_{f \in \F}(\RH_{\S'}(f) -\RH_\S(f)) \right]  \\
&= \E_{\S'} \left[ \sup_{f \in \F}
\left( \frac{1}{n} 
\sum_{i \in [n]}(\ell(y'_i,f(x'_i))-\ell(y_i,f ( x_i )))\right) \right]  \\
&= \frac{1}{n} \E_{\S'} \left[  \sup_{f \in \F}
\left(\sum_{j=1}^J w_j\sum_{i \in I_j}(\ell(y'_i,f(x'_i))-\ell(y_i,f ( x_i )))\right) \right],
\end{align*}
where the second equality is due to the decomposition \eqref{decomposition}. 

Then by the sub-additivity of the supremum, we have
\begin{align*}
\sup_{f \in \F} (R(f)-\RH_\S(f))
\le g(\S),
\end{align*}
where \red{$g(\S)$ is defined by
\begin{align*}
g(\S)
= \frac{1}{n} \E_{\S'} \left[  
\sum_{j=1}^J w_j 
\sup_{f \in \F} \left( \sum_{i \in I_j}(\ell(y'_i,f(x'_i))-\ell(y_i,f ( x_i )))\right) \right],
\end{align*}
and satisfies $g(\S)
= \sum_{j} w_j g_{j} ( \S )$, 
where for each $j$, 
\begin{align*}
g_j( \S ):=~\frac{1}{n} \E_{\S'} 
\left[ \sup_{f \in \F} \left(\sum_{i \in I_j}(\ell(y'_i,f(x'_i))-\ell(y_i,f ( x_i )))\right) \right].
\end{align*}
Note that each function $g_j$ 
has bounded difference $M/n$
and satisfies \eqref{decom},
and therefore is a decomposable Lipschitz function.}
Then using Theorem \ref{thm:FracConBounds},
for all $\delta\in (0,1)$, with probability at least $1-\delta$, we have
\begin{align*}
&\sup_{f \in \F} (R(f)-\RH_\S(f))
\le 
\E_{\S} [ g(\S) ] + M \sqrt{\dfrac{\chi_f(G)}{2n}\Log{\dfrac{1}{\delta}}} \\
&= \sum_{j=1}^J \frac{w_j}{n} 
\E_{\S, \S'} \left[ \sup_{f \in \F} 
\left( \sum_{i \in I_j}(\ell(y'_i,f(x'_i))-\ell(y_i,f ( x_i )))
\right)
\right]
+M\sqrt{\dfrac{\chi_f(G)}{2n}\Log{\dfrac{1}{\delta}}}.
\end{align*}
\red{
Note that
\begin{multline*}
\E_{\S, \S'} \left[ \sup_{f \in \F} 
\left( \sum_{i \in I_j}(\ell(y'_i,f(x'_i))-\ell(y_i,f ( x_i )))
\right)
\right] \\
= \E_{\S, \S'} \E_{\boldsymbol{\sigma}} 
    \left[ \sup_{f \in \F} \left( \sum_{i \in I_j}\sigma_i(\ell(y'_i,f(x'_i))-\ell(y_i,f ( x_i ))) \right) 	\right],
\end{multline*}
since the introduction of Rademacher variables $\boldsymbol{\sigma}=(\sigma_i)_i$, 
uniformly taking values in $\{-1,+1\}$, 
does not change the expectation.}
Indeed, for $\sigma_i=+1$, 
the corresponding summand stays unaltered, 
and for $\sigma_i=-1$, 
the corresponding summand reverses sign,
which is the same as flipping $(x_i,y_i)$ and $(x'_i,y'_i)$ between $\S$ and $\S'$. 
This change has no effect on the overall expectation 
as we are considering the expectation over $\S$ and $\S'$,
\red{and by noting that $\S$ and $\S'$ are independent and $I_{j}$ is some independent set.}
Therefore, we have
\begin{align}
    &\quad\sup_{f \in \F} (R(f)-\RH_\S(f)) \nonumber \\ 
    &\le \sum_{j=1}^J \frac{w_j}{n} \E_{\S, \S'} \E_{\boldsymbol{\sigma}} 
    \left[ \sup_{f \in \F} \left( \sum_{i \in I_j}\sigma_i(\ell(y'_i,f(x'_i))-\ell(y_i,f ( x_i ))) \right) 	\right]
    + M\sqrt{\dfrac{\chi_f(G) }{2n} \Log{\dfrac{1}{\delta}} } \notag\\
	&\le 2 \sum_{j=1}^J \frac{w_j}{n} 
    \E_{\boldsymbol{\sigma}} \E_{\S}
    \left[ \sup_{f \in \F} \left( \sum_{i \in I_j}\sigma_i(\ell(y_i,f ( x_i ))) \right) 	\right]
    + M\sqrt{\dfrac{\chi_f(G) }{2n} \Log{\dfrac{1}{\delta}} },
\label{sym}
\end{align}
where the last step uses the sub-additivity of the supremum.
Then in view of Definition \ref{df:FracCover} of $\widehat{\mathfrak R}_{\S}^{\star}$,
we obtain
\begin{align*}
    \sup_{f \in \F} (R(f)-\RH_\S(f)) 
    \le 2\E_{\S} \left[ \widehat{\mathfrak R}_{\S}^{\star}(\ell\circ\F) \right]
    +M\sqrt{\dfrac{\chi_f(G) }{2n}\Log{\dfrac{1}{\delta}}}.
\end{align*}
Therefore the first bound \eqref{eq:FracRadGenBounds} follows from the definition of the supremum,
that is, for all $f\in \F$,
\[
R(f) - \RH_\S(f) \le \sup_{f \in \F} \left(R(f)-\RH_\S(f)\right).
\]

\redd{
Note that
\begin{align*}
\widehat{\mathfrak R}_{\S}^{\star}(\ell\circ\F) 
= \sum_{j} w_j
\left( \frac{1}{\mra{n}} \E_{\boldsymbol{\sigma}} 
\left[ 
\sup_{f \in \F}
\left( \sum_{i \in I_j}\sigma_i(\ell(y_i,f ( x_i ))) \right)
\right]
\right)
\end{align*}
satisfies the condition of Theorem \ref{thm:FracConBounds}
with bounded difference $M/n$,
and therefore concentrates around its expectation 
$\mathfrak R^{\star}(\ell\circ\F)$.
Then using the union bound with \eqref{eq:FracRadGenBounds},
yields the second bound \eqref{eq:EmpFracRadGenBounds}.}
\end{proof}

From Remark \ref{RadComiid}, 
Theorem \ref{thm:GenBounds} is a natural extension of the standard Rademacher generalization bounds 
when examples are identically and independently distributed 
(see, for example, \citealt[Theorem 3.3]{mohri2018foundations}), 
as in this case, $\chi_f(G)=1$.

\begin{re}
\redd{
To use the symmetrization technique in equation \eqref{sym}, 
the variables involved in the same summation need to be independent.   
Consequently, when extending the concept of Rademacher complexities to scenarios involving interdependent variables, it becomes necessary to decompose the set of random variables into independent sets. In this context, the fractional independent vertex cover $\{ ( I_{j}, w_{j} ) \}_{j}$
with $\sum_{k} w_{k} = \chi_f(G)$
emerges as a pivotal tool for achieving an optimal decomposition,
\red{as $\chi_f(G)$ is the minimum of $\sum_{k} w_{k}$ over all fractional independent vertex covers.}
}
\end{re}
\subsection{Generalization bounds via algorithmic stability}

This section establishes stability bounds for learning from graph-dependent data, using the concentration inequalities derived in the last section.
Algorithmic stability 
has been used in the study of classification and regression to derive generalization bounds ~\citep{rogers1978finite,devroye1979distribution,kearns1999algorithmic,kutin2002almost}.
A key advantage of stability bounds is that they are designed for specific learning algorithms, 
exploiting particular properties of the algorithms. 

Since uniform stability was introduced in \cite{bousquet2002stability},
it has been among the most widely used notions of algorithmic stability.
Given a training sample $\S$ of size $n$,
for every $i\in [n]$, 
removing the $i$-th element from $\S$ 
results in a sample of size $n-1$, which is denoted by
\begin{align*}
\S^{\setminus i} 
:= ( (x_1, y_1), \ldots, (x_{i - 1}, y_{i - 1}), (x_{i + 1}, y_{i + 1}) \ldots, (x_n, y_n)).
\end{align*}

A learning algorithm $\A$ is a function that maps the training set $\S$ 
onto a function $f^{\A}_{\S}: \mathcal X \rightarrow \mathcal Y$.

\begin{df}[Uniform stability,~\citealt{bousquet2002stability}]
Given an integer $n > 0$, the learning algorithm $\A$ is $\beta_n$-uniformly stable 
with respect to the loss
function $\ell$, 
if for any $i\in [n]$, $\S \in (\mathcal X \times \mathcal Y)^n$, and $(x, y) \in \mathcal X \times \mathcal Y$,
it holds that
\begin{align}
\mid  \ell(y, f^{\A}_{\S}(x)) - \ell(y, f^{\A}_{\S^{\setminus i}}(x))  \mid  \le \beta_n.
\label{stable}
\end{align}
\end{df}
Intuitively, 
small perturbations of the training sample have little effect on the learning for a stable learning algorithm.

Now, we begin our analysis by considering \red{the difference between the empirical error 
and the generalization error of a learning algorithm $f^{\A}_{\S}$ 
trained over a $G$-dependent sample $\S$,
formally defined by
\begin{align}
\Phi_{\A}(\S) := R(f^{\A}_{\S}) - \RH_\S (f^{\A}_{\S}).
\label{phi}
\end{align}
The mapping $\Phi_{\A}: (\mathcal X \times \mathcal Y)^n \rightarrow \R$ 
will play a critical role in estimating $R(f^{\A}_{\S})$ via stability. 
We will first bound the probability of the deviation of $\Phi_{\A}(\S)$ from its expectation 
(Lemma \ref{stability concentration}), 
and then obtain an upper bound of expected value of $\Phi_{\A}(\S)$ (Lemma \ref{staExp}).} 



\begin{lm}
Given a $G$-dependent sample $\S$ of size $n$, 
and a $\beta_n$-uniformly stable learning algorithm $\A$.
Suppose the loss function $\ell$ is bounded by $M$.
Then for any $t>0$,
\[
\P( \Phi_{\A}(\S) - \E [\Phi_{\A}(\S)] \ge t )
\le \exp \left( - \frac{2n^2t^2}{\Lambda(G) (4n\beta_n + M)^2} \right).
\]
\label{stability concentration}
\end{lm}

We prove the following lemma, which states that 
the Lipschitz coefficients of $\Phi_{\A}(\cdot)$ are all bounded by $4\beta_n + M/n$.
Then Lemma~\ref{stability concentration} follows from
Lemma~\ref{generalization error bound bounded difference}
and Theorem~\ref{ie-dep-graph},
since the Lipschitz coefficients are all of the same value.

\begin{lm}
Given a $\beta_n$-uniformly stable learning algorithm $\A$,
for any $\S, \S' \in (\mathcal X \times \mathcal Y)^n$ that differ only in one entry, 
we have
\[
\mid  \Phi_{\A}(\S) - \Phi_{\A}(\S') \mid
\,\le 4\beta_n + \frac{M}{n}.
\]
\label{generalization error bound bounded difference}
\end{lm}

\begin{proof}
In the literature \cite{bousquet2002stability}, 
Lemma~\ref{generalization error bound bounded difference}
was proved for \red{the} i.i.d. case, actually,
the proof remains valid in our dependent setting.
Assume that $\S$ and $\S'$ differ only in $i$-th entry,
and denote $\S'$ as
\[
\S^i
:= ( (x_1, y_1), \ldots, (x_{i - 1}, y_{i - 1}), 
(x_i', y_i'), (x_{i + 1}, y_{i + 1}) \ldots, (x_m, y_m)),
\]
such that the marginal distribution of $(x_i', y_i')$ is also $\D$.

Notice that we do not require the data to be i.i.d., 
as samples are dependent,
and have the same marginal probability distribution $\D$.
To begin with, we bound $R (f^{\A}_{\S}) - R (f^{\A}_{\S^i})$ \redd{using the triangle inequality},

\begin{align*}
&\big\mid  R (f^{\A}_{\S}) - R (f^{\A}_{\S^i}) \big\mid 
\le \big\mid  R (f^{\A}_{\S}) - R ( f^{\A}_{\S^{\setminus i}} ) \big\mid  
+ \big\mid  R(f^{\A}_{\S^{\setminus i}}) - R (f^{\A}_{\S^i}) \big\mid  
\\
&= \big\mid  \E_{\D} [\ell(y, f^{\A}_{\S}(x))] - \E_{\D} [\ell(y, f^{\A}_{\S^{\setminus i}}(x))] \big\mid  
+ \big\mid  \E_{\D} [\ell(y, f^{\A}_{\S^{\setminus i}}(x))] - \E_{\D} [ \ell(y, f^{\A}_{\S^i}(x))] \big\mid 
\\
&= \big\mid  \E_{\D} [ \ell(y, f^{\A}_{\S}(x)) - \ell(y, f^{\A}_{\S^{\setminus i}}(x)) ] \big\mid  
+ \big\mid  \E_{\D} [\ell(y, f^{\A}_{\S^{\setminus i}}(x)) - \ell(y, f^{\A}_{\S^i}(x))] \big\mid  
\le 2\beta_n,
\end{align*}
\redd{where the last inequality is by the uniform stability defined by \eqref{stable}.}

Then we bound $\RH_\S (f^{\A}_{\S}) - \RH_{\S^i} (f^{\A}_{\S^i}) $,
\begin{align*}
&n \big\mid  \RH_\S (f^{\A}_{\S}) - \RH_{\S^i} (f^{\A}_{\S^i}) \big\mid  
= \left\mid  \sum_{ (x_j, y_j) \in \S } \ell(y_j, f^{\A}_{\S}(x_j))
- \sum_{ (x_j, y_j) \in \S^i } \ell(y_j, f^{\A}_{\S^i}(x_j)) \right\mid  \\
&\le \big\mid  \ell(y_i, f^{\A}_{\S}( x_i )) - \ell(y_i^\prime, f^{\A}_{\S^i}(x_i^\prime)) \big\mid 
+ \sum_{j\ne i} \left\mid \ell(y_j, f^{\A}_{\S}(x_j)) - \ell(y_j, f^{\A}_{\S^i}(x_j)) \right\mid  
\\
&\le \sum_{j\ne i} \left\mid \ell(y_j, f^{\A}_{\S}(x_j)) - \ell(y_j, f^{\A}_{\S^{\setminus i}}(x_j)) \right\mid  
+ \sum_{j\ne i} \left\mid  \ell(y_j, f^{\A}_{\S^{\setminus i}}(x_j)) - \ell(y_j, f^{\A}_{\S^i}(x_j)) \right\mid  \\ 
&\quad+ \big\mid  \ell(y_i, f^{\A}_{\S}( x_i )) - \ell(y_i^\prime, f^{\A}_{\S^i}(x_i^\prime)) \big\mid 
\,\le 2n\beta_n + M,
\end{align*}
\redd{where the last inequality 
is by the uniform stability 
and the assumption that $\ell$ is bounded by $M$.
}

Combining the above bounds,
by the triangle inequality, 
we have that
\begin{align*}
\big\mid  \Phi_{\A}(\S) - \Phi_{\A}(\S^i)\big\mid 
&= \big\mid  ( R (f^{\A}_{\S}) - \RH_\S (f^{\A}_{\S}) ) - ( R (f^{\A}_{\S^i}) - \RH_{\S^i} (f^{\A}_{\S^i}) ) \big\mid  \\
&\le \big\mid  R (f^{\A}_{\S}) - R (f^{\A}_{\S^i}) \mid  + \mid  \RH (f^{\A}_{\S}) - \RH_{\S^i} (f^{\A}_{\S^i}) \big\mid  
\le 4\beta_n + \dfrac{M}{n},
\end{align*}
\red{which} completes the proof.
\end{proof}


\red{We are now in measure to bound the expectation of $\Phi_{\A}(\S)$}.
\begin{lm}
Let $\S$ be a $G$-dependent sample of size $n$.
Suppose the maximum degree of $G$ is $\Delta = \Delta(G)$.
Let $\A$ be a $\beta_i$-uniformly stable learning algorithm for every $i \in [n-\Delta, n]$,
and $\beta_{n, \Delta} = \max_{i \in [0,\Delta]} \beta_{n-i}$. 
Then we have
\begin{align*}
\E [ \Phi_{\A}(\S) ] \le 2 \beta_{n, \Delta} (\Delta + 1).
\end{align*}
\label{staExp}
\end{lm}
The proof of the lemma is based on iterative perturbations of the training sample $\S$,
where a perturbation is essentially removing a data point from $\S$. 
The property of uniform stability of the algorithm guarantees that each perturbation causes a discrepancy up to $\beta_{n,\Delta}$, 
and in total $2(\Delta +1)$ perturbations have to be made to 
\textit{eliminate} the dependency between a data point and the others. 

We start with a technical lemma before the proof of Lemma~\ref{staExp}.
\begin{lm}
Under the same assumptions in Lemma \ref{staExp},
we have
\[
\max_{(x_i, y_i) \in \S}
\, \E_{(x, y), \S} [\ell(y, f^{\A}_{\S}(x)) - \ell(y_i, f^{\A}_{\S}( x_i ))]
\le 2 \beta_{n, \Delta}(\Delta+1).
\]
\label{pert}
\end{lm}
\begin{proof}
For every $i\in [n]$, 
let $N_G(i)$ be the set of vertices adjacent to $i$ in graph $G$,
and suppose $N_G^+(i) = N_G(i) \cup \{ i \} = \{j_1,\ldots,j_{n_i}\}$ with $j_{k-1}>j_k$. 
Define $\S^{(i,0)}=\S$ and for every $k \in [n_i]$, 
let
$\S^{(i,k)}$ be obtained from $\S^{(i,k-1)}$ by removing the $j_k$-th entry. 
By the uniform stability of $\A$, 
for any $(x,y)\in \mathcal X \times\mathcal Y$ and $k\in [n_i]$, 
we have
\begin{align*}
\big\mid \ell(y, f^{\A}_{\S^{(i,k-1)}}(x))-\ell(y, f^{\A}_{\S^{(i,k)}}(x))\big\mid 
\le \beta_{n, \Delta}.
\end{align*}

By a decomposition using a telescoping summation,
\begin{align*}
\ell(y, f^{\A}_{\S}(x)) 
= \sum_{k=1}^{n_i} (\ell(y, f^{\A}_{\S^{(i,k-1)}}(x))-\ell(y, f^{\A}_{\S^{(i,k)}}(x)) 
+ \ell(y, f^{\A}_{\S^{(i,n_i)}}(x)).
\end{align*}
Similarly, we also get
\begin{align*}
\ell(y_i, f^{\A}_{\S}( x_i )) 
= \sum_{k=1}^{n_i} (\ell(y_i, f^{\A}_{\S^{(i,k-1)}}( x_i ))-\ell(y_i, f^{\A}_{\S^{(i,k)}}( x_i )) 
+ \ell(y_i, f^{\A}_{\S^{(i,n_i)}}( x_i )).
\end{align*}
Now we are ready to bound the difference
\begin{align*}
&\quad\ell(y, f^{\A}_{\S}(x)) - \ell(y_i, f^{\A}_{\S}( x_i ))  \\
&= \sum_{k=1}^{n_i} \left( (\ell(y, f^{\A}_{\S^{(i,k-1)}}(x))-\ell(y, f^{\A}_{\S^{(i,k)}}(x)))
- (\ell(y_i, f^{\A}_{\S^{(i,k)}}( x_i ))-\ell(y_i, f^{\A}_{\S^{(i,k-1)}}( x_i ))) \right)  \\
&\quad+ \ell(y, f^{\A}_{\S^{(i,n_i)}}(x))-\ell(y_i, f^{\A}_{\S^{(i,n_i)}}( x_i )) \\
&\le \sum_{k=1}^{n_i} \mid  \ell(y, f^{\A}_{\S^{(i,k-1)}}(x))-\ell(y, f^{\A}_{\S^{(i,k)}}(x)) \mid  \\
&\quad+ \sum_{k=1}^{n_i} \mid  \ell(y_i, f^{\A}_{\S^{(i,k)}}( x_i ))-\ell(y_i, f^{\A}_{\S^{(i,k-1)}}( x_i )) \mid 
+ \ell(y, f^{\A}_{\S^{(i,n_i)}}(x))-\ell(y_i, f^{\A}_{\S^{(i,n_i)}}( x_i ))\\
&\le 2 n_i\beta_{n, \Delta} + \ell(y, f^{\A}_{\S^{(i,n_i)}}(x))-\ell(y_i, f^{\A}_{\S^{(i,n_i)}}( x_i )).
\end{align*}
\redd{Therefore, by noting that 
$n_i = \mid N_G^+(i) \mid\,\le \Delta+1$ for all $i$, we have}
\begin{align*}
&\E_{\S, (x, y)} [\ell(y, f^{\A}_{\S}(x)) - \ell(y_i, f^{\A}_{\S}( x_i ))] \\
&\le \E_{\S, (x, y)} [ \ell(y, f^{\A}_{\S^{(i,n_i)}}(x)) - \ell(y_i, f^{\A}_{\S^{(i,n_i)}}( x_i )) ] 
+ 2 n_i \beta_{n, \Delta} \\
&\le \E_{\S, (x, y)} [ \ell(y, f^{\A}_{\S^{(i,n_i)}}(x)) - \ell(y_i, f^{\A}_{\S^{(i,n_i)}}( x_i )) ] 
+ 2 \beta_{n, \Delta}(\Delta+1) \\
&= \E_{\S, (x, y)} [ \ell(y, f^{\A}_{\S^{(i,n_i)}}(x)) ] 
- \E_\S [\ell(y_i, f^{\A}_{\S^{(i,n_i)}}( x_i )) ]
+ 2 \beta_{n, \Delta}(\Delta+1) \\
&= \E_{\S^{(i,n_i)}, (x, y)} [ \ell(y, f^{\A}_{\S^{(i,n_i)}}(x)) ]
- \E_{\S^{(i,n_i)},(x_i,y_i)} [\ell(y_i, f^{\A}_{\S^{(i,n_i)}}( x_i )) ] 
+ 2 \beta_{n, \Delta}(\Delta+1) \\
&= 2 \beta_{n, \Delta}(\Delta+1),
\end{align*}
where the last equality is because
$(x_i, y_i)$ and $(x, y)$ are independent of $\S^{(i,n_i)}$, and have the same distribution.
\end{proof}

Now we are ready to prove Lemma~\ref{staExp}.
\begin{proof}[Proof of Lemma~\ref{staExp}]
From the definition of $\Phi_{\A}(\S)$ in \eqref{phi}, we have
\begin{align*}
\E_{\S} [ \Phi_{\A}(\S) ]
&= \E_{\S} 
\left[ \E_{(x, y)} [\ell(y, f^{\A}_{\S}(x))] - \frac{1}{n} \sum_{i = 1}^n \ell(y_i, f^{\A}_{\S}( x_i )) \right] \\
&= \frac{1}{n} \sum_{i = 1}^n \E_{\S, (x, y)}
[ \ell(y, f^{\A}_{\S}(x)) - \ell(y_i, f^{\A}_{\S}( x_i )) ] 
\le 2 \beta_{n, \Delta}(\Delta+1),
\end{align*}
where the last inequality is by Lemma \ref{pert}.
\end{proof}

Combining Lemma~\ref{stability concentration} and Lemma~\ref{staExp}
gives the following theorem, 
\red{which upper-bounds the generalization error of learning algorithms trained over $G$-dependent training sets of size $n$.}
\begin{tm}
Let $\S$ be a sample of size $n$ with dependency graph $G$.
Suppose the maximum degree of $G$ is $\Delta$.
Assume that the learning algorithm $\A$ is $\beta_i$-uniformly stable for all $i \in [n - \Delta, n]$.
Suppose the 
loss function $\ell$ is bounded by $M$. 
Let $\beta_{n, \Delta} = \max_{i \in [0,\Delta]} \beta_{n-i}$.
For any $\delta \in (0, 1)$, with probability at least $1 - \delta$,
it holds that
\begin{align*}
R(f^{\A}_{\S}) 
\le \RH_\S (f^{\A}_{\S}) + 2 \beta_{n, \Delta} (\Delta + 1)
+ \frac{4n\beta_n + M}{n} \sqrt{\dfrac{\Lambda(G)}{2} \Log{\dfrac{1}{\delta}} }.
\end{align*}
\label{stabBounds}
\end{tm}

\begin{re}
It is well known that for many learning algorithms, $\beta_n=O(1/n)$
(see, for example, \citealt{bousquet2002stability}),
in this case, we have that
$\beta_{n, \Delta} (\Delta + 1) 
\le \beta_{n-\Delta} (\Delta + 1) = O(\frac{\Delta}{n - \Delta})$,
which vanishes asymptotically if $\Delta = o(n)$.
The term $O\left(\sqrt{\Lambda(G)}/n\right)$ also vanishes asymptotically if 
$\Lambda(G)=o(n^2)$. 
\red{We also observe that if the training data are i.i.d., 
Theorem~\ref{stabBounds} degenerates to the standard stability bound obtained in~\cite{bousquet2002stability}, 
by setting $\Delta = 0$, $\beta_{n, \Delta} = \beta_n$, and $\Lambda(G) = n$.}

\label{gen bound analysis}
\end{re}

\section{Applications}
\label{sec:4}
\mra{In this section, we present three practical 
applications related to learning with interdependent data, 
for which we use the methodology presented in the previous sections 
to derive generalization bounds.}

\subsection{Bipartite ranking}

The goal of bipartite ranking is to assign higher scores to instances of the positive class than the ones of the negative class \citep{Freund03,Agarwal09}. This framework corresponds to many applications of information retrieval such as recommender systems \citep{Sidana21}, and uplift-modeling \citep{Betlei21}, etc.

It has attracted a lot of interest in recent years 
since the empirical ranking error of a scoring function $h:\mathcal{X}\rightarrow \mathbb{R}$ 
over a training set $T :=\left(x_i,y_i\right)_{1\le i\le m}$ with $y_i\in\{-1,+1\}$ defined by
\begin{equation}
\label{eq:BiparLoss}
\widehat{\mathcal{L}}_T(h)=\frac{1}{m_{-}m_{+}} \sum_{i:y_i=1}\sum_{j:y_j=-1}\I{h( x_i )\le h(x_j)},
\end{equation}
is equal to one minus the Area Under the ROC Curve (AUC) of $h$ 
(see, for example, \citealt{Cortes04}),
where $m_- :=\sum_{i=1}^m \I{y_i=-1}$ and $m_+ :=\sum_{i=1}^m \I{y_i=1}$ 
are the number of negative and positive instances in the training set $T$ respectively.

For two instances of different classes $(x,y), (x',y')$ 
in $T$ such that $y\neq y'$,
by considering the (unordered) pairs of examples 
$\{ (x,y), (x',y') \}$,
and the classifier of pairs $f$ associated to a scoring function $h$ defined by
\[
f(x,x')=h(x)-h(x'),
\]
\redd{we can rewrite the bipartite ranking loss \eqref{eq:BiparLoss} of $h$ over $T$ as the classification error of the associated $f$ over the pairs of instances of different classes,
\begin{equation}
\label{eq:BiparLoss2}
\widehat{R}_{\S}(f)
= \widehat{\mathcal{L}}_T(h)
= \frac{1}{n} \sum_{\{ (x, y), (x', y') \} \in\S}\I{z_{y,y'}f(x,x')\le 0},
\end{equation}
where $n=m_{-}m_{+}$, 
\begin{align*}
\S := \left\{ (x, y), (x',y') : 
(x, y) \in T, (x',y') \in T, y\neq y' \right\}    
\end{align*}
is the set of $n$ unordered pairs of examples from different classes in $T$, 
and 
\begin{align*}
    z_{y,y'} :=2\I{y-y'>0}-1.
\end{align*}
Note that $z_{1,-1} = 1$ and $z_{-1,1} = -1$.
}

Let 
\begin{align*}
T^+ := \{ ( x^+_i, 1): i\in[m_+]\} 
\quad\mbox{and}\quad
T^- := \{ (x^-_j, -1): j\in[m_-]\}
\end{align*}
be the sets of positive and negative instances of $T$ respectively. 
Then $T=T^+\cup T^-$. 
Without loss of generality, we assume that $m_+ \le m_-$, 
which corresponds to the usual situation in information retrieval, 
where there are fewer positive (relevant) instances than negative (irrelevant) ones. 

In this case, the independent covers of 
the corresponding dependency graph of $\S$ is
$\{ ( I_k, 1 ) \}_{k\in\{1,\ldots,m_-\}}$, where
\[
I_k = 
\left\{\left(x_i^+,x^-_{\sigma_{k,m_-}(i)}\right)
: i\in [m_+]\right\}, 
\]
with $\sigma_{k,m_-}$ denoting the permutation that is defined by
\begin{align*}
\sigma_{k,m_-}(i)=\begin{cases}(k+i-1) (\mathrm{mod}\ m_-), &\text{if } (k+i-1)(\mathrm{mod}\ m_-)\neq 0 \\
m_-, & \text{otherwise.}\end{cases}
\end{align*}

Figure \ref{fig:bipartRank} illustrates 
the dependency graph of a bipartite ranking problem with $m_+=2$ positive examples and $m_-=3$ negative instances 
as well as its corresponding independent covers represented by dotted ellipsoids.

\begin{figure}[H]
\centering
\scalebox{.75}{
\begin{tikzpicture}
\node[rectangle,draw=black,fill=white, minimum size=1.4cm] (11) at (-10, -1) {$x_2^+$};
\node[rectangle,draw=black,fill=white, minimum size=1.4cm] (12) at (-10, 1) {$x_1^+$};

\node[circle,draw=black,fill=white, minimum size=1.6cm] (21) at (-7.5, -2) {$x_3^-$};
\node[circle,draw=black,fill=white, minimum size=1.6cm] (22) at (-7.5, 0) {$x_2^-$};
\node[circle,draw=black,fill=white, minimum size=1.6cm] (23) at (-7.5, 2) {$x_1^-$};

\draw (11) -- (21);
\draw (11) -- (22);
\draw (11) -- (23);

\draw (12) -- (21);
\draw (12) -- (22);
\draw (12) -- (23);

\node[] at (-4.2,3.5) {{\large$I_1$}};
\node[] at (5.5,-2.5) {{\large$I_2$}};
\node[] at (-5.5,-2.5) {{\large$I_3$}};

\node[circle,draw=black,fill=white] (x12) at (0: 4cm) {$(x_1^+, x_2^-)$};
\node[circle,draw=black,fill=white] (x11) at (120: 4cm) {$(x_1^+, x_1^-)$};
\node[circle,draw=black,fill=white] (x13) at (240: 4cm) {$(x_1^+, x_3^-)$};

\node[circle,draw=black,fill=white] (x21) at (180: 4cm) {$(x_2^+, x_1^-)$};
\node[circle,draw=black,fill=white] (x23) at (300: 4cm) {$(x_2^+, x_3^-)$};
\node[circle,draw=black,fill=white] (x22) at (60: 4cm) {$(x_2^+, x_2^-)$};

\draw (x11) -- (x12);
\draw (x11) -- (x13);
\draw (x12) -- (x13);

\draw (x21) -- (x22);
\draw (x21) -- (x23);
\draw (x22) -- (x23);

\draw (x21) -- (x11);
\draw (x22) -- (x12);
\draw (x23) -- (x13);

\node[line width=0.5mm, dotted, ellipse, draw=black, rotate=0, scale=.9] [fit=(x11) (x22)] {};
\node[line width=0.5mm, dotted, ellipse, draw=black, rotate=-30, scale=.78] [fit=(x12) (x23)] {};
\node[line width=0.5mm, dotted, ellipse, draw=black, rotate=30, scale=.78] [fit=(x21) (x13)] {};

\end{tikzpicture}
}

\caption{The graph on the right
is a dependency graph corresponding to a bipartite ranking problem with $m_+=2$ positive examples $T^+=\{ (x_1^+, 1), (x_2^+, 1) \}$; and $m_-=3$ negative ones, $T^-=\{ (x_1^-, -1), (x_2^-, -1), (x_3^-, -1)\}$. 
Each pair of examples from different classes 
corresponds to an edge of the complete bipartite graph $K_{2, 3}$ on the left,
and is represented by a vertex of the dependency graph on the right.
Two pairs are adjacent in the dependency graph if they have an example in common.
Fractional independent covers $\{ (I_k, 1) \}_{1\le k\le 3}$ are shown by dotted ellipsoids.}
\label{fig:bipartRank}
\end{figure}

\rui{
\begin{re}
In the bipartite ranking,
the dependent pairs of instances correspond to
the edges of a complete bipartite graph $K_{m_+, m_-}$,
since pairs are chosen with one positive instance and one negative instance, see Figure \ref{fig:bipartRank} for illustration.

Given a graph $G$, the line graph of $G$ has the edges of $G$ as its vertices,
with two vertices adjacent if the corresponding edges have a vertex in common in $G$.
Then the dependency graph for pairs $\S$
is the line graph of $K_{m_+, m_-}$, 
known as an $m_+ \times m_-$ Rook's graph,
which is a Cartesian product of two complete graphs.
\end{re}
}

\redd{For bipartite ranking, 
it is easy to check that
\begin{align*}
\frac{\chi_f}{n}
= 
\frac{\max(m_-,m_+)}{ m_-m_+ }
= \frac{1}{ \min(m_-,m_+) }.    
\end{align*}}
Therefore by Theorem \ref{thm:FracRad}, Theorem \ref{thm:GenBounds},
and Ledoux and Talagrand's contraction lemma \citep[p.78 Corollary 3.17]{LT91} 
that can be extended to fractional Rademacher complexities 
giving 
$\widehat{\mathfrak{R}}^\star_\S(\mathcal{\ell\circ\mathcal{F}})=2\widehat{\mathfrak{R}}^\star_S(\mathcal{\mathcal{F}})$, 
we can bound the generalization error of bipartite ranking as follows.

\begin{co}
Let $T$ be a training set composed of $m_+$ positive instances and $m_-$ negative ones; and $\S$ the set of unordered pairs of examples from different classes in $T$.
Then for any scoring function from 
$\F=\{f:(x,x')\mapsto \langle \mathbf{w},\phi(x)-\phi(x')\rangle: \|\mathbf{w}\|\le B\}$, 
where $\phi$ is a feature mapping with bounded norm such that $\|\phi(x)-\phi(x')\|\le \Gamma$ for all $(x,x')$,
and for any $\delta\in (0,1)$, 
with probability at least $1-\delta$, we have 
\[
R(f) \le \RH_\S (f) + \frac{4B\Gamma}{\sqrt{m}}
+ 3\sqrt{\dfrac{ 1}{2m} \Log{\dfrac{2}{\delta} } },
\]
where $m = \min(m_-,m_+)$.
\label{cor:bipartite}
\end{co}

\subsection{Multi-class classification}

We now address the problem of mono-label multi-class classification, 
where the output space is a discrete set of labels  $\mathcal{Y} = [K]$ with $K$ classes. 
For the sake of presentation, we denote an element of $\mathcal{X} \times \mathcal{Y}$ as $x^y := (x, y)$. 
For a class of predictor functions $\mathcal{H} = \{h : \mathcal{X} \times \mathcal{Y} \rightarrow \R\}$, 
let $\ell$ be the instantaneous loss of $h\in \mathcal{H}$ on example $x^y$ defined by
\[
\ell(y,h(x^y))=\frac{1}{K-1}
\sum_{y'\in \mathcal{Y}\setminus \{y\}}\I{h(x^y)\le h(x^{y'})}.
\]
\red{For any sample $x$, this loss function 
is the average number of classes, 
for which $h$ assigns a higher score to the pairs constituted by $x$ and any other classes
that are not the true class of $x$.}
For a training set $T=\left(x_i^{y_i} \right)_{1\le i\le m}$ of size $m$, 
the corresponding empirical error of a function $h\in \mathcal{H}$ is
\begin{align}
\widehat{\mathcal{L}}_T(h)
= \frac{1}{m(K-1)}\sum_{i=1}^m\sum_{y'\in \mathcal{Y}\setminus \{y_i\}}\I{h(x_i^{y_i})\le h(x_i^{y'})}.
\label{mccLoss}
\end{align}
Many multi-class classification algorithms like Adaboost.MR \citep{Schapire99} or the multiclass SVM \citep{weston98} aim to minimize a convex surrogate function of this loss.

Similar to the bipartite ranking case, by considering pairs $(x^y,x^{y'})$ with $y'\in\mathcal{Y}\setminus \{y\}$, 
constituted by the pairs $x^y$ of an example and its class,
and the pairs $x^{y'}$ of the same examples with all other classes,
the classifier of pairs $f$ associated to a function $h\in\mathcal{H}$
is defined by
\[
f(x^y,x^{y'})=h(x^y)-h(x^{y'}).
\]
Then the empirical loss of a function $h$ over $T$, 
can be written as the classification error of the associated $f$,
\begin{equation}
\label{eq:mulLoss2}
\widehat{R}_{\S}(f)
= \widehat{\mathcal{L}}_T(h)
= \frac{1}{n} \sum_{(x^y,x^{y'})\in\S}\I{z_{y,y'}f(x^y,x^{y'})\le 0},
\end{equation}
where $\S=\{(x^y,x^{y'}) :  x^y\in T, x^{y'} \in T, y\neq y'\}$
is of size $n=m(K-1)$, 
and $z_{y,y'}=2\I{y>y'}-1$. 
In this case, an independent cover of the corresponding dependency graph of $\S$ could be
$\{ ( I_k, 1 ) \}_{ k\in\{1,\ldots,K-1\} }$, where
\[
I_k=\left\{\left(x_i^1,x_i^{k+1}\right): i\in [m]\right\},
\]
with the corresponding fractional chromatic number $\chi_f=K-1$.

\begin{figure}[htb] 
\centering
\scalebox{0.7}{
\begin{tikzpicture}

\node[] at (-2.2,2) {{\large$I_1$}};
\node[] at (-.8,0) {{\large$I_2$}};
\node[] at (-2.2,-2) {{\large$I_3$}};

\node[circle,draw=black,fill=white,inner sep=0pt] (x11) at (0.25, 2) {$(x_1^1, x_1^2)$};
\node[circle,draw=black,fill=white, inner sep=0pt] (x12) at (2, 0) {$(x_1^1, x_1^3)$};
\node[circle,draw=black,fill=white, inner sep=0pt] (x13) at (0.25, -2) {$(x_1^1, x_1^4)$};

\node[circle,draw=black,fill=white, inner sep=0pt] (x21) at (3.75, 2) {$(x_2^1, x_2^2)$};
\node[circle,draw=black,fill=white, inner sep=0pt] (x22) at (5.75, 0) {$(x_2^1, x_2^3)$};
\node[circle,draw=black,fill=white, inner sep=0pt] (x23) at (3.75, -2) {$(x_2^1, x_2^4)$};

\node[circle,draw=black,fill=white, inner sep=0pt] (x31) at (7.75, 2) {$(x_3^1, x_3^2)$};
\node[circle,draw=black,fill=white, inner sep=0pt] (x32) at (9.75, 0) {$(x_3^1, x_3^3)$};
\node[circle,draw=black,fill=white, inner sep=0pt] (x33) at (7.75, -2) {$(x_3^1, x_3^4)$};

\node[circle,draw=black,fill=white, inner sep=0pt] (x41) at (11.75, 2) {$(x_4^1, x_4^2)$};
\node[circle,draw=black,fill=white, inner sep=0pt] (x42) at (13.75, 0) {$(x_4^1, x_4^3)$};
\node[circle,draw=black,fill=white, inner sep=0pt] (x43) at (11.75, -2) {$(x_4^1, x_4^4)$};

\draw (x11) -- (x12);
\draw (x11) -- (x13);
\draw (x12) -- (x13);

\draw (x21) -- (x22);
\draw (x21) -- (x23);
\draw (x22) -- (x23);

\draw (x31) -- (x32);
\draw (x31) -- (x33);
\draw (x32) -- (x33);

\draw (x41) -- (x42);
\draw (x41) -- (x43);
\draw (x42) -- (x43);

\node[line width=.5mm, dotted, ellipse, draw=black, scale=.84] [fit=(x11) (x21) (x31) (x41)] {};
\node[line width=.5mm, dotted, ellipse, draw=black, scale=.84] [fit=(x12) (x22) (x32) (x42)] {};
\node[line width=.5mm, dotted, ellipse, draw=black, scale=.84] [fit=(x13) (x23) (x33) (x43)] {};

\end{tikzpicture}
}
\caption{\redd{The dependency graph for the multi-class classification problem
with $m=4$ examples and $K=4$ classes
is a vertex-disjoint union of $4$ trianlges.}
Fractional independent covers $\{ (I_k, 1) \}_{1\le k\le 3}$ are shown by dotted ellipsoids.}
\label{fig:Multiclass}
\end{figure}

Figure \ref{fig:Multiclass} illustrates a dependency graph for the multi-class 
classification problem with $m=4$ and $K=4$ 
as well as 
the corresponding fractional independent covers represented by dotted ellipsoids.  

Similar to the bipartite ranking case, 
we have the following corollary based on the prior results.

\begin{co}
Let $T$ be a training set of $K$-label instances and of size $m$.
Let $\S$ be the set of no-redundant pairs of examples from different classes in $T$.
Then for any scoring functions from 
$\F=\{f:(x,x')\mapsto \langle \mathbf{w},\phi(x)-\phi(x')\rangle: \|\mathbf{w}\|\le B\}$, 
where $\phi$ is a feature mapping with the bounded norm such that $\|\phi(x)-\phi(x')\|\le \Gamma$ for all $(x,x')$, and for any $\delta\in (0,1)$, 
with probability at least $1-\delta$, we have 
\[
R(f) \le \RH_\S (f) + \frac{4B\Gamma}{\sqrt{m}}+ 3\sqrt{\dfrac{ 1}{2m} \Log{\dfrac{2}{\delta}}}.
\]
\label{cor:mulcls}
\end{co}

\mra{
\begin{re}
The loss function we considered \eqref{mccLoss} is normalized by $K-1$, and we obtain a result that is comparable to the binary classification case.
For a loss function based on margins, $\ell(y,h(x^y))=h(x^y)-\displaystyle \max_{y'\neq y}h(x^{y'})$; the Rademacher complexity term grows in lockstep with the number of classes $K$.  
\end{re}
}

\subsection{Learning from $m$-dependent data}

Here \mra{we consider learning from $m$-dependent data, 
and give a practical learning scenario}.
Suppose that there are linearly aligned locations, 
for example, real estate along a street. 
Let $y_i$ be the observation at location $i$, 
for example, the house price. 
Let $x_i$ denote the random variable modeling geographical effect at location $i$. 
Assume that $x$'s are mutually independent and each $y_i$ is geographically influenced by a neighborhood of size at most $2q+1$. 
The goal is to learn to predict $y$ from a sample 
$ \{( ( x_{i-q},\ldots,x_i,\ldots,x_{i+q} ), y_i )\}_{i\in [n]}$, where $n$ is the size of the sample.
See Figure \ref{fig:m-dep}
for an example.

\begin{figure}[htb]
\begin{center}
\begin{tikzpicture}[scale=.85]
\node at (-1.5, 0) {$\ldots$};
\node[circle,draw=black,fill=white,minimum size=1cm] (A) at (0, 0) {\normalsize$x_{i-2}$};
\node[circle,draw=black,fill=white,minimum size=1cm] (B) at (2, 0) {\normalsize$x_{i-1}$};
\node[circle,draw=black,fill=white,minimum size=1cm] (C) at (4, 0) {\normalsize$x_{i}$};
\node[circle,draw=black,fill=white,minimum size=1cm] (D) at (6, 0) {\normalsize$x_{i+1}$};
\node[circle,draw=black,fill=white,minimum size=1cm] (E) at (8, 0) {\normalsize$x_{i+2}$};
\node[circle,draw=black,fill=white,minimum size=1cm] (F) at (10, 0) {\normalsize$x_{i+3}$};
\node at (11.5, 0) {$\ldots$};

\node at (2.5, -2.5) {$\ldots$};
\node[rectangle,draw=black,fill=white,minimum size=.8cm] (X) at (4, -2.5) {\normalsize$y_i$};
\node[rectangle,draw=black,fill=white,minimum size=.8cm] (Y) at (6, -2.5) {\normalsize$y_{i+1}$};
\node at (7.5, -2.5) {$\ldots$};

\draw[-] 
(X) edge (A)
(X) edge (B)
(X) edge (C)
(X) edge (D)
(X) edge (E);

\draw[-]
(Y) edge (B)
(Y) edge (C)
(Y) edge (D)
(Y) edge (E)
(Y) edge (F);
\end{tikzpicture}
\end{center}
\medskip
\caption{Each observation $y_i$ is geographically determined by a set of variables 
$\{ x_j \}_{i -2 \le j \le i + 2}$ of size $5$.
The sample $ \{( \{ x_j \}_{i -2 \le j \le i + 2}, y_i )\}_{i}$ is $4$-dependent.}
\label{fig:m-dep}
\end{figure}

This model accounts for the impact of local locations on house prices. 
Similar scenarios are frequently considered in spatial econometrics,
and \redd{moving average processes in time series analysis},
see~\cite{anselin2013spatial} for more examples.

The above application is a special case of $m$-dependence.
A sequence of random variables $\{ X_i \}_{i = 1}^n$ is said to be $f(n)$-dependent 
if subsets of variables separated by some distance greater than $f(n)$ are independent.
This model was introduced by  \cite{hoeffding1948central}
and has been studied extensively (see, for example, \citealt{stein1972bound, chen1975poisson}).
This is usually the canonical application for the results based on the dependency graph model.
A special case of $f(n)$-dependence when $f(n) = m$ is the following $m$-dependent model.

\begin{df}[$m$-dependence,~\citealt{hoeffding1948central}]
A sequence of random variables $\{ X_i \}_{i = 1}^n$ is $m$-dependent for some $m \ge 1$
if $\{ X_j \}_{j = 1}^i$ and $\{ X_j \}_{j = i+m+1}^n$
are independent for all $i  > 0$.
\label{m-dependence result}
\end{df}

\begin{figure}[bth] 
\centering
\begin{subfigure}[c]{.6\textwidth}
\begin{tikzpicture}[scale=.8]
\node[] (0) at (-1, 0) {$G$};
\node[circle,draw=black,fill=white] (A) at (0, 0) {};
\node[circle,draw=black,fill=white] (B) at (1.5, 0) {};
\node[rectangle,draw=black,fill=white] (C) at (3, 0) {};
\node[rectangle,draw=black,fill=white] (D) at (4.5, 0) {};
\node[regular polygon,regular polygon sides=3,draw,scale=0.6,fill=white] (E) at (6, 0) {};
\node[regular polygon,regular polygon sides=3,draw,scale=0.6,fill=white] (F) at (7.5, 0) {};
\draw[out=-60, in=-120]  
(A) edge (B) (B) edge (C) (C) edge (D) (D) edge (E) (E) edge (F);
\draw[out=-90, in=-90]
(A) edge (C) (B) edge (D) (C) edge (E) (D) edge (F);
\begin{pgfonlayer}{background}
\draw[edgeBIG] (0, 0) edge (1.5, 0);
\draw[edgeBIG] (3, 0) edge (4.5, 0);
\draw[edgeBIG] (6, 0) edge (7.5, 0);
\end{pgfonlayer}
\end{tikzpicture}
\end{subfigure}
%
%

\vspace{.5cm}

\begin{subfigure}[c]{.6\textwidth}
\begin{tikzpicture}[scale=.6]
\node[] (0) at (-4.35, 0) {$F$};
\node[circle,draw=black,fill=white] (A) at (0, 0) {};
\node[rectangle,draw=black,fill=white] (C) at (2, 0) {};
\node[regular polygon,regular polygon sides=3,draw,scale=0.6,fill=white] (E) at (4, 0) {};
\draw[-]  
(A) edge (C) (C) edge (E);
\end{tikzpicture}
\end{subfigure}
\caption{\redd{A tree-partition of 
the dependency graph for $2$-dependent variables.}}
\label{chain}
\end{figure}

Figure~\ref{chain} illustrates a dependency graph $G$ for a 2-dependent sequence $\{ X_i \}_{i}$, and its tree-partition. The illustration demonstrates the division of an $m$-dependent sequence into blocks of size $m$. Subsequently, these blocks are sequentially mapped to vertices of a path of length $\left\lceil n/m \right\rceil$, as depicted in Figure~\ref{chain}. This tree-partition shows that $\Lambda(G) \le \left(\left\lceil n/m \right\rceil- 1\right)( m+m )^2 + m^2  \le 4mn + m^4$.

Combining Theorem~\ref{stabBounds} and the above estimate of forest complexity
gives the following.

\begin{co}
Let $\S$ be an $m$-dependent sample of size $n$.
Assume that the learning algorithm $\A$ is $\beta_i$-uniformly stable for any 
$i \in [n - 2m, n]$.
Suppose the loss function $\ell$ is bounded by $M$. 
For any $\delta \in (0, 1)$,
with probability at least $1 - \delta$,
it holds that
\[
R(f^{\A}_{\S}) 
\le \RH_\S (f^{\A}_{\S}) + 2 \beta_{n, 2m} (2m + 1)
+ ( 4n\beta_n + M ) \sqrt{\dfrac{2 m }{ n } 
\left( 1 + \dfrac{m}{n} \right) \Log{\dfrac{1}{\delta}} }.
\]
\label{learning m dependent}
\end{co}


Choose any uniformly stable learning algorithm in~\cite{bousquet2002stability} with $\beta_n = O(1/n)$,
such as regularization algorithms in RKHS, etc., 
and apply to the above-mentioned house price prediction problem. 
Then for any fixed $q$, with high probability, Corollary~\ref{learning m dependent} 
gives that
$R(f^{\A}_{\S}) 
\le \RH_\S (f^{\A}_{\S}) + O\left(\sqrt{\dfrac{ 1}{ n } \Log{\dfrac{1}{\delta}} }\right)$ 
for sufficiently large $n$, matching the stability bound in the i.i.d.\ case in \cite{bousquet2002stability}.


\section{Concluding remarks}
\label{sec:5}
In this survey, we presented various McDiarmid-type concentration inequalities 
for functions of graph-dependent random variables.
These concentration bounds were then used to 
obtain generalization error bounds for learning from graph-dependent samples 
via fractional Rademacher complexity and algorithm stability.

\red{
We also included
some real practical applications of the methodology.
Note that
in our applications, 
the sample contains dependent data with the same marginal distribution,
but this is not necessary 
and concentration inequalities derived are without this assumption,
and therefore can be applied to situations
where the distribution may change over time.
}

\red{
The dependency graphs used for our applications 
exhibit certain structural regularities
and therefore we have explicit simple bounds.
For applications under various other settings,
we can still obtain meaningful bounds
as long as we have suitable estimates of the fractional chromatic number or forest complexity.
We will leave interested readers to investigate and find more applications.
}

There are various new directions that can be explored.

\begin{enumerate}[leftmargin=.5cm, label=\arabic*.]

\item For dependent data, there are other definitions of the generalization error, such as the one specified in \cite{mohri2008stability,mohri2010stability,kuznetsov2017generalization}.
The connections between these and the one we used have been discussed in 
\citep{mohri2008stability,mohri2010stability}.
It is a natural question whether
our results can be adapted to this definition.

\item 
The dependency graph model we consider
requires variables in disjoint non-adjacent subgraphs to be independent.
There are some newly introduced dependency graph models such as
weighted dependency graphs~\citep{dousse2019weighted,feray2018weighted},
and the combination of mixing coefficients and dependency graphs \citep{lampert2018dependency, isaev2021extremal}.
It would be interesting to use these new dependency graphs to obtain generalization bounds for learning under different dependent settings.

\item 
Recently, there are some new breakthroughs establishing sharper
stability bounds \citep{feldman2019high, bousquet2020sharper}.
It would be interesting 
to follow these results 
and to obtain sharper stability bounds
for learning under graph-dependence.


\end{enumerate}

\section*{Acknowledgments}

R.-R. Z. thanks David Wood for email communications on tree-partitions.
The authors are sincerely grateful to the referees for carefully reading
the manuscript and providing invaluable comments and suggestions, 
which led to a substantial improvement in the presentation.

%
%



\end{document}